%% file: ex_article.tex
\documentclass[onefignum,onetabnum]{siamonline220329}

\input{ex_shared}
\input{Defs}

\ifpdf
\hypersetup{
  pdftitle={Understanding Overfitting in Adversarial Training in Kernel Regression},
  pdfauthor={T.Zhang and K. Li}
}
\fi

\begin{document}

\maketitle

\begin{abstract}
Adversarial training and data augmentation with noise are widely adopted techniques to enhance the performance of neural networks. 
This paper investigates adversarial training and data augmentation with noise in the context of regularized regression in a reproducing kernel Hilbert space (RKHS). We establish the limiting formula for these techniques as the attack and noise size, as well as the regularization parameter, tend to zero. Based on this limiting formula, we analyze specific scenarios and demonstrate that, without appropriate regularization, these two methods may have larger generalization error and Lipschitz constant than standard kernel regression. However, by selecting the appropriate regularization parameter, these two  methods can outperform standard kernel regression and achieve smaller generalization error and Lipschitz constant. These findings support the empirical observations that adversarial training can lead to overfitting, and appropriate regularization methods, such as early stopping, can alleviate this issue. 


\end{abstract}



\section{Introduction}
Over the past decade, deep learning, a subfield of artificial intelligence that utilizes neural network-based models, has undergone a major transformation, leading to significant advancements in various domains such as computer vision, speech recognition, and robotics~\cite{goodfellow2016deep}.  As a result, a substantial body of research has emerged aimed at explaining the mathematical foundations of deep learning's success \cite{Barron1993,bach2017breaking,pmlr-v70-raghu17a,10.5555/3295222.3295232,DBLP:conf/iclr/DuZPS19,DBLP:conf/iclr/LampinenG19,45820,e2020mathematical,bartlett_montanari_rakhlin_2021}.


Recent studies have demonstrated that neural network models, despite achieving human-level performance on many important tasks, are not robust to adversarial examples—a small and human imperceptible input perturbation can easily change the prediction label. In light of the vulnerability of standard neural network algorithms to adversarial attacks, several algorithms have been proposed that aim to provide resilience to such attacks~\cite{papernot16a,Sharif16, Papernot16, Szegedy13,papernot15,Carlini17, Madry2017TowardsDL, Biggio18, Sethi18, Ji18,zheng2016improving,novak2018sensitivity,DBLP:journals/corr/abs-2005-10247,DBLP:conf/iclr/MadryMSTV18,DBLP:conf/ndss/YuYZTHJ20}. For example, adversarial training \cite{DBLP:journals/corr/abs-2005-10247,DBLP:conf/iclr/MadryMSTV18} involves incorporating the generation of adversarial examples into the training process in order to improve the worst-case performance of the model.
 However, the theoretical foundations of these robust algorithms are not well understood, and adversarial training can present practical challenges such as implicit bias \cite{Li2020Implicit} and overfitting \cite{rice2020}. Regularization techniques, such as early stopping, have been employed to mitigate the overfitting issue \cite{chen2021robust}, but the underlying mechanism is  unclear.

This research aims to investigate the overfitting problem in adversarial training from a theoretical perspective and explains why regularization techniques can be used to mitigate it. Motivated by the observation that neural networks can be approximated with neural tangent kernel regression \cite{DBLP:journals/corr/abs-1806-07572,10.5555/3454287.3455454}, we analyze adversarial training and data augmentation with noise in the context of kernel regression in reproducing kernel Hilbert space (RKHS). The present study contributes to the field by:
\begin{itemize}
\item Establishing the limiting formula for adversarial training and data augmentation with noise as the regularization parameter and the size of the attack or noise approach zero.
\item Demonstrating that in comparison with standard kernel regression, adversarial training and data augmentation without regularization have greater generalization error and Lipschitz constant, whereas adversarial training and data augmentation with suitable regularization have lower generalization error and Lipschitz constant.
\item Verifying theoretical findings by experiments on artificial and real data sets.
\end{itemize}
 Our theoretical analysis, in conjunction with the neural tangent kernel, provides a first explanation as to why adversarial training tends to exhibit overfitting, and why regularization techniques are effective in mitigating this issue. 

\subsection{Existing works}
This section presents a comprehensive review of the key developments in the field of neural network-based machine learning, with an emphasis on the tools and results that are relevant to our proposed research.

\textbf{A Mathematical Understanding of Neural Network-Based Machine Learning}
Neural network-based models have become increasingly popular for addressing the problem of supervised learning, where a training dataset $S={(\mathbf{x}_i,y_i)}_{i=1}^n\subset \mathbb{R}^{d}\times \mathbb{R}$ is provided and the goal is to find a model $\hat{f}: \mathbb{R}^d\rightarrow\mathbb{R}$ that generalizes well, meaning that $\hat{f}(\mathbf{x})$ provides accurate predictions of $y$ on subsequent $(\mathbf{x},y)$ pairs. A commonly used measure to measure the goodness of $\hat{f}$ is the generalization error $
\calL(f)=\Expect_{(\bx,y)} l(f(\bx),y)$, where $l$ is a suitable loss function such as the squared loss.

In neural network-based machine learning, the model $\hat{f}$ is selected from a family of highly nonlinear statistical models. As a prototypical example, the feed-forward neural network with $L$ layer can be parametrized as $f(\bx;\btheta)$ by $f(\bx;\btheta)=\sigma_L(\bW_L\sigma_{L_1}(\bW_{L-1}\cdots\sigma_1(\bW_1\bx))),$ where $\btheta=(\bW_1,\cdots,\bW_L)$ in the set of parameters with $\bW_l\in\reals^{d_l\times d_{l-1}}$, and $\sigma_l:\reals^{d_l}\rightarrow \reals^{d_{l-1}}$ are nonlinear activation functions.  In practice, the fitted function $\hat{f}$, parameterized by $f(\bx;\hat{\theta})$, is  usually chosen by minimizing the empirical risk function 
\begin{equation}\label{eq:objective1}\calL_n(\btheta)=\frac{1}{n}\sum_{i=1}^nl(y_i,f(\bx_i;\btheta))\end{equation} through gradient descent or stochastic gradient descent.

\textbf{Tangent kernel of neural networks}
The neural tangent kernel (NTK) is a valuable instrument for comprehending neural network training and implicit regularization in gradient descent. 
It is shown that the evolution of neural networks during their training by gradient descent can be described using kernel regression with the neural tangent kernel (NTK) \cite{10.1145/3406325.3465355,DBLP:conf/iclr/DuZPS19,pmlr-v97-allen-zhu19a}. In particular, in any model that $f$ can be parameterized by $f(\cdot,\theta)$ with $\theta\in\reals^m$, the NTK kernel $K$ is defined by \begin{equation}\label{eq:NTK}K(\bx,\by)=\sum_{i=1}^m\partial_{\theta_i}f(\bx;\theta)\partial_{\theta_i}f(\by;\theta).\end{equation}
It is shown that as the width of a network goes to infinity, the gradient algorithm that minimizes \eqref{eq:objective1} can be approximated with regression in reproducing kernel Hilbert space (RKHS) with the neural tangent kernel  in \eqref{eq:NTK} \cite{DBLP:journals/corr/abs-1806-07572}.

%

\textbf{Adversarial training} 
Considering that it has been observed that neural network-based algorithm are vulnerable to attacks, there have been numerous works that attempt to locate such adversarial example when the model is given, and on the other hand, develop robust deep learning algorithms.

In general, a defense mechanism called adversarial training~\cite{DBLP:journals/corr/GoodfellowSS14,madry2018towards} generates robust algorithms by modifying the objective function that considers the worst-case
perturbations and taking  the gradient descent steps at (approximate) worst-case perturbations \cite{goodfellow2014explaining,kurakin2017adversarial,madry2018towards}, or by using provably upper bound inner maximization  \cite{NEURIPS2018_358f9e7b,pmlr-v80-mirman18b,raghunathan2018certified,pmlr-v97-cohen19c}. For example, instead of minimizing the standard objective function \eqref{eq:objective1}, these works minimize a modified objective function 
\begin{equation}\label{eq:objective2}
\calL_n^{(adv)}(\btheta)=\frac{1}{n}\sum_{i=1}^n\max_{\bdelta\in\bDelta}l(y_i,f(\bx_i+\bdelta;\btheta)),
\end{equation}
where $\bDelta$ is the set of perturbations, such as $\bDelta=\{\bDelta: \|\bDelta\|\leq \epsilon\}$.  In terms of implementation, \cite{10.5555/3454287.3455454,NEURIPS2020_0740bb92} show that in the over-parameterized regime, a heuristic form of robust optimization that alternates between minimization and maximization steps converges to a
solution where the training error is within $\epsilon$ of
the optimal robust loss. However, it is unclear whether this algorithm can achieve a low generalization error.


\textbf{Data augmentation} Another commonly used strategy for improving robustness applies  the technique of data augmentation~\cite{10.5555/3327345.3327439,9585649}, which enhances the size of the training datasets by adding noises or attacks artificially and minimize
\begin{equation}\label{eq:objective3}\calL_n^{(aug)}(\btheta)=\frac{1}{n}\sum_{i=1}^n\mathrm{Average}_{\bdelta\in\bDelta}l(y_i,f(\bx_i+\bdelta;\btheta)),\end{equation}
where $\bDelta$ is the set of noises or attacks. 
As pointed out by Goodfellow et al. in \cite[Section 7.4-7.5]{goodfellow2016deep}, injecting noise in the input to a neural network is a form of data augmentation, and can be considered as a regularization strategy.
Empirical studies \cite{DBLP:journals/corr/abs-2111-05328} verify its effectiveness in increasing the robustness of deep neural network methods.


\textbf{Existing theories for adversarial learning and data augmentation} While the adversarial training method has been popular, there has only been a few theoretical studies on it.  
 Li et al. \cite{Li2020Implicit} shows that the  implicit bias of gradient descent-based adversarial training  promotes robustness against adversarial perturbation.  For the basic setting of linear regression, Javanmard et al. \cite{DBLP:conf/colt/JavanmardSH20}  precisely characterize the standard and robust accuracy and the corresponding tradeoff achieved by a contemporary mini-max adversarial training approach in a high-dimensional regime where the number of data points and the parameters of the model grow in proportion to each other. However, the generalization to nonlinear kernels is not considered in this work. \cite{yu2021exploring} studies adversarial learning, but only in the setting of a few learning algorithms such as $k$-nn. 

The theoretical analysis of data augmentation with noise is also quite lacking and existing works mostly focus on a few specific settings. Chen et al.~\cite{NEURIPS2020_f4573fc7} develop a framework to explain data augmentation as averaging over the orbits of the group that keeps the data distribution approximately invariant, and show that it leads to variance reduction. However, this analysis does not apply to the ``noise injection'' approach. 
Dao et al.~\cite{pmlr-v97-dao19b} use the connection between data augmentation and kernel classifiers to show that a kernel classifier on augmented data approximately decomposes into two components: (i) an averaged version of the transformed features, and (ii) a data-dependent variance regularization term. This suggests that data augmentation improves generalization both by inducing invariance and by reducing model
complexity. Hanin and Sun~\cite{hanin2021how} analyze the effect of augmentation on optimization in the simple convex setting of linear regression with MSE loss, and interpret augmented (S)GD as a stochastic optimization method for a time-varying sequence of proxy losses.

\textbf{Overfitting in adversarial neural network framework} 
It has been observed for adversarial deep learning, overfitting to the training set does in fact harm robust performance to a very large degree in adversarially robust training across multiple datasets \cite{rice2020}: after a
certain point in adversarial training, further training will continue to substantially decrease the robust training loss while increasing the robust test loss. In practice, regularization techniques such as early stopping can mitigate the overfitting problem \cite{Li2020Implicit}. However, the justification for overfitting in adversarial learning is missing. Interestingly, this is different from the traditional neural networks, where 
overfitting does not harm the performance~\cite{zhang2017understanding,Bartlett30063,10.1214/19-AOS1849}, and it has been theoretically justified in some settings \cite{https://doi.org/10.48550/arxiv.2202.05928,https://doi.org/10.48550/arxiv.2202.06526}.


\subsection{Outline of the paper}
The paper is organized as follows. 
Section~\ref{sec:background} presents an introduction to the background and the problem setup that is being considered in this research. Section~\ref{sec:main} provides a rigorous derivation of the limiting formula for the data augmented and adversarial estimators. Section~\ref{sec:examples} examines several scenarios and demonstrates that, in these cases, the ridgeless versions of these estimators exhibit overfitting behavior, as measured by both the generalization error and the Lipschitz constant, and the overfiting behavior can be ameliorated by regularization techniques. Section~\ref{sec:simulations} verifies this phenomenon using both artificial and real data sets.

\section{Background}\label{sec:background}
This work studies the problem of supervised learning: 
assume $n$ i.i.d. observed pairs $(\bx_i, y_i)$, $1\leq i\leq n$, drawn from an unknown probability distribution $\mu(\bx, y)$, and $f^*$ is the conditional expectation function $f^*(\bx) = \Expect(y|\bx = \bx)$, then we aim to find $f$ such that the generalization risk 
\[
\calL({f})=\Expect_{\bx\sim \mu} ({f}(\bx)-f^*(\bx))^2
\]
is as small as possible.

The supervised learning framework finds the function $f^*$ from a set  $\mathcal{H}$, and here we consider the setting where $\mathcal{H}$ is a reproducing kernel Hilbert space (RKHS), a set of functions $\reals^p\rightarrow\reals$ that is associated with a positive
definite kernel $K(\cdot, \cdot ): \reals^p \times \reals^p \rightarrow \reals$ and endowed with the norm $\|\cdot\|_{\calH}$. Then we consider the standard estimator \eqref{eq:objective1}, with an additional ridge regularization term $\lambda\|f\|_{\calH}^2$ that is commonly used in regression in RKHS:
\begin{align}\label{eq:traditional}
\hat{f}_{\lambda}&=\argmin_{f\in\calH}\hat{\calL}_{\lambda}(f),\,\,\,\text{where}\,\,\,\hat{\calL}^{(adv)}_{\lambda}(f)=\frac{1}{n}\sum_{i=1}^n(y_i-f(\bx_i))^2+\lambda\|f\|_{\calH}^2.
\end{align}

In addition, we consider the adversarial estimator \eqref{eq:objective2} and the  estimator with data augmentation \eqref{eq:objective3} in the setting of RKHS, also with additional regularization: 
\begin{align}\label{eq:adversarial}
\hat{f}_{\lambda}^{(adv)}&=\argmin_{f\in\calH}\hat{\calL}^{(adv)}_{\lambda}(f),\,\,\,\text{where}\,\,\,\hat{\calL}^{(adv)}_{\lambda}(f)=\frac{1}{n}\sum_{i=1}^n\max_{\bdelta\in\bDelta}(y_i-f(\bx+\bdelta))^2+\lambda\|f\|_{\calH}^2,\\\label{eq:augment}
\hat{f}_{\lambda}^{(aug)}&=\argmin_{f\in\calH}\hat{\calL}^{(aug)}_{\lambda}(f),\,\,\,\text{where}\,\,\,\hat{\calL}^{(aug)}_{\lambda}(f)=\frac{1}{n}\sum_{i=1}^n\Expect_{\bdelta\in\bDelta}(y_i\!-\!f(\bx_i+\bdelta))^2+\lambda\|f\|_{\calH}^2.
\end{align}
When $\lambda=0$, the estimators above become ``ridgeless'' and reduce to the estimators \eqref{eq:objective1}, \eqref{eq:objective2}, and \eqref{eq:objective3}. 

\textbf{Background on RKHS and Notations}  RKHS $\calH$ is an inner product space that consists of function from $\reals^p$ to $\reals$. For any $\bx\in\reals^p$, there exists $K_{\bx}\in \calH$ such that $f(\bx)=\langle K_{\bx},f\rangle_{\calH}$ holds for all $f\in\calH$. It then follows that $K(\bx,\bx')=\langle K_{\bx}, K_{\bx'} \rangle_{\calH}$ and $\|f\|_{\calH}=\sqrt{\langle f, f \rangle_{\calH}}$. 

We assume that the mapping $\bx\rightarrow K_{\bx}$ is continuous and differentiable. In addition, we let $\bT_{\bx}: \reals^p\rightarrow\calH$ be the coefficient of the first-order expansion of $K_{\bx}$, and 
\[
\lim\sup_{\bdelta\rightarrow 0}\frac{\|K_{\bx+\bdelta}-(K_{\bx}+\bT_{\bx}\bdelta)\|_{\calH}}{\|\bdelta\|^2}
\]
is bounded for all $\bx\in\reals^p$. 

In addition, we let $\bK_{\bX}=[K_{\bx_1},\cdots,K_{\bx_n}]\in\reals^{\dim(\calH)\times n}$ and $\bT_{\bX}=[\bT_{\bx_1},\cdots,\bT_{\bx_n}]\in\reals^{\dim(\calH)\times np}$, and use $\bP_{K_{\bX}}, \bP_{K_{\bX}^\perp}:  \calH\rightarrow \calH$ to represent projectors to the subspace spanned by  $\bK_{\bX}$ and its orthogonal subspace respectively. We also use $\Span(\bK_{\bX},\bT_{\bX})$ to denote the subspace spanned by the columns of $\bK_{\bX}$ and $\bT_{\bX}$.



Throughout the paper, we employ the notations $C$ and $c$ to represent constants that are solely dependent on $\bX$ and the kernel $K$, and remain independent of both $\epsilon$ and $\lambda$. We note that the specific value of these constants may vary across different equations.

\section{Main result}\label{sec:main}
This section establishes the limiting formula for the estimators $\hat{f}^{(aug)}_{\lambda}$ and $\hat{f}^{(adv)}_{\lambda}$ as the regularization parameter $\lambda$ and the set $\bDelta$ both go to zero. This ``small set'' assumption is consistent with the empirical implementation where the added noise in data augmentation and the neighborhood in adversarial training are usually small.  In particular, Section~\ref{sec:aug} investigates the estimator with data augmentation and Section~\ref{sec:adv} investigates the adversarial estimator.

For the analysis in this section, we assume $\bDelta=\epsilon\bDelta_0$ for some fixed $\bDelta_0$ sand let $\epsilon,\lambda\rightarrow 0$ simultaneously. We note that $\bDelta_0$ and $\bDelta$ could be either discrete sets or continuous distributions. In the following paper, we use $\Expect_{\delta\in\bDelta}f(\delta)$ to represent the average of $f(\delta)$ over the discrete set $\bDelta$; or the expectation when $\bDelta$ is a distribution.

\subsection{Deterministic result for estimator with data  augmentation}\label{sec:aug}
This section establishes the formula for the augmented estimator $\hat{f}^{(aug)}_{\lambda}$, by showing that it is close to 
\begin{equation}\label{eq:auxillary3_lemma}
g_{\lambda}=\hat{f}_\lambda+\Big(\bP_{\bK_{\bX}^\perp}^T\bSigma\bP_{\bK_{\bX}^\perp}+\frac{\lambda}{\epsilon^2}\bI\Big)^{-1}\bP_{\bK_{\bX}^\perp}^T\bSigma\bP_{\bK_{\bX}}\hat{f}_\lambda,
\end{equation}
for $\bSigma=\frac{1}{n}\sum_{i=1}^n\bT_{\bx_i}(\Expect_{\bdelta\in\bDelta_0}\bdelta\bdelta^T)\bT_{\bx_i}^T$. To understand $g_\lambda$, let us investigate a few cases: 
\begin{itemize}
\item $g_{0}=\hat{f}_0$.
\item When $\lambda=o(\epsilon^2)$, then $g_{\lambda}=\hat{f}_{\lambda}+o(1)=\hat{f}_{0}+o(1)$.
\item When $\lambda/\epsilon^2\rightarrow\infty$, we have
\begin{equation}\label{eq:g0}
g_{\lambda}=\hat{f}_{\lambda}+\Big(\bP_{\bK_{\bX}^\perp}^T\bSigma\bP_{\bK_{\bX}^\perp}\Big)^{-1}\bP_{\bK_{\bX}^\perp}^T\bSigma\bP_{\bK_{\bX}}\hat{f}_{\lambda}+o(1).
\end{equation}
Geometrically, $g_{\lambda}$ can be understood as the function in $\calH$ such that it is orthogonal to $\bT_{\bX}$ and its projection to $\bK_{\bX}$ is $\hat{f}_\lambda$.

If in addition, $\lambda=o(1)$, then we have $\hat{f}_{\lambda}=\hat{f}_0+o(1)$, so $g_{\lambda}=g_0+o(1)$.
\item When $\lambda$ is in the order of $\epsilon^2$, then it can be considered as an interpolation between $\hat{f}_0$ and $g_0$. 
\end{itemize}
Part (a) of the following theorem shows that when $\lambda$ is in the range of $(C\epsilon^3,\infty)$ for some large $C$, then $g_{\lambda}$ is a good approximation of $\hat{f}^{(aug)}_\lambda$ in the sense that $\Big\|\hat{f}^{(aug)}_\lambda-g_{\lambda}\Big\|_{\calH}=o(1)$. Part (b) shows that when $\lambda$ is small or zero, $g_{0}$ is a good approximation of the projection of $\hat{f}^{(aug)}_{\lambda}$ to the subspace $\Span(\bK_{\bX},\bT_{\bX})$.

\begin{thm}\label{thm:aug}[Limiting formula for  augmented estimator]
(a) Assume that $\Expect_{\bdelta\in\bDelta}\bdelta=0$, then there exists $\epsilon_0$ depending on the kernel $K$ and the data set $\bX$ such that for all $\epsilon\leq\epsilon_0$,
\begin{equation}\label{eq:aug1}
\Big\|\hat{f}^{(aug)}_\lambda-g_{\lambda}\Big\|_{\calH}\leq C(\epsilon^3/\lambda+\epsilon).
\end{equation}

(b) There exists $c>0$ such that for $\lambda\leq c\epsilon^2$, then \[\Big\|P_{\Span(\bK_{\bX},\bT_{\bX})}\hat{f}^{(aug)}_\lambda-g_0 \Big\|_{\calH}\leq C(\sqrt{\epsilon^2+\lambda/\epsilon^2}).\]
\end{thm}


\textbf{Outline of the proof} 
The proof is based on an auxiliary estimator, whose objective function is a linear approximation of the function in \eqref{eq:augment}: 
\begin{align}\nonumber
\tilde{f}^{(aug)}_{\lambda}&=\argmin_{f\in\calH}\tilde{\calL}^{(aug)}_{\lambda}(f),\,\,\,\text{where}\\\tilde{\calL}^{(aug)}_{\lambda}(f)&=\frac{1}{n}\sum_{i=1}^n\Expect_{\bdelta\in\bDelta}(y_i-\langle f,K_{\bx_i}+\bT_{\bx_i}\bdelta\rangle)^2+\lambda\|f\|_{\calH}^2.\label{eq:auxillary2}
\end{align}
The proof of  Theorem~\ref{thm:aug} is divided into two components. First, we will show that $\tilde{f}^{(aug)}_{\lambda}$ and $\hat{f}^{(aug)}_{\lambda}$ are close; second, $\tilde{f}^{(aug)}_{\lambda}$ and $g_{\lambda}$ are close. The proof is based on various techniques in linear algebra and deferred to Section~\ref{sec:proof}.

\begin{remark}[Constants in Theorem~\ref{thm:aug}]
From the proof we note that $\epsilon_0$ is in the order of and $C$ in \eqref{eq:aug1} is in the order of 
\[
O(\epsilon^4\|\bT_{\bX}\|^4/\lambda+\epsilon^2\|\bT_{\bX}\|^2)/\sigma_{\min}^2(K_{\bX})+O(\epsilon^3/\lambda+\epsilon)/\sigma_{\min}(K_{\bX}+\epsilon \bT_{\bX})
\]
and $\epsilon_0$ is chosen such that
\[
\epsilon_0\leq c\min(\bT_{\bX}+\epsilon_0\bK_{\bX})/\epsilon_0, \,\, \epsilon_0^2\|\bT_{\bX}\|^2+\epsilon_0^4\|\bT_{\bX}\|^4/\lambda\leq c\sigma_{\min}^2(\bK_{\bX}).
\]\end{remark}

\subsection{Deterministic result for adversarial estimator}\label{sec:adv}
This section establishes the limiting formula for the adversarial estimator $\hat{f}^{(adv)}_\lambda$ in \eqref{eq:adversarial}, under the  assumption that $\bDelta_0$ is a unit ball, which is one of the most popular settings in adversarial training. 

In particular, we will show that similar to the augmented estimator $\hat{f}^{(aug)}_\lambda$ in Theorem~\ref{thm:aug}, $\hat{f}^{(adv)}_\lambda$ is also approximated by \begin{equation}\label{eq:auxillary3_lemma2}
g_{\lambda}'=\hat{f}_\lambda+\Big(\bP_{\bK_{\bX}^\perp}^T\bSigma'\bP_{\bK_{\bX}^\perp}+\frac{\lambda}{\epsilon^2}\bI\Big)^{-1}\bP_{\bK_{\bX}^\perp}^T\bSigma\bP_{\bK_{\bX}}\hat{f}_\lambda,
\end{equation}
for $\bSigma'=\frac{1}{n}\sum_{i=1}^n\bT_{\bx_i}\bT_{\bx_i}^T$. Note that $g_{\lambda}'$ is $g_{\lambda}$ with $\Expect_{\bdelta\in\bDelta_0}\bdelta\bdelta^T$ replaced with $\bI$.

In particular, when $\lambda$ is in the range of $(C\epsilon^{2.5},\epsilon/C)$ for some large $C$, then $g_{\lambda}'$ is a good approximation of $\hat{f}^{(adv)}_\lambda$ in the sense that $\Big\|\hat{f}^{(adv)}_\lambda-g_{\lambda}'\Big\|_{\calH}=o(1)$. In addition, part (b) shows that in the setting where $\lambda=0$ or small, $g_0'(\bx)$ is an approximation of $\hat{f}^{(aug)}_\lambda(\bx)$ for $\bx$ in the subspace $\Span(\bK_{\bx},\bT_{\bx})$.

\begin{thm}\label{thm:adv}[Limiting formula for the adversarial estimator]
(a) There exists $\epsilon_0,c>0$ such that for $\epsilon\leq \epsilon_0$ and $\lambda\leq c\epsilon$,
\[
\Big\|\hat{f}^{(adv)}_\lambda-g_{\lambda}'\Big\|_{\calH}\leq C(\epsilon^{2.5}/\lambda+\lambda).
\]
(b) There exists $c>0$ such that for $\lambda\leq c \epsilon^2$, we have $\Big\|P_{\Span(\bK_{\bX},\bT_{\bX})}\hat{f}^{(aug)}_\lambda-g_0' \Big\|_{\calH}\leq C(\sqrt{\epsilon^2+\lambda/\epsilon^2})$.
\end{thm}

\textbf{Outline of the proof} We investigate the following auxiliary estimator, whose objective function is a linear approximation of $\hat{\calL}^{(adv)}$ in \eqref{eq:auxillary2}:
\begin{align}\label{eq:auxillary}
&\tilde{f}^{(adv)}_{\lambda}=\argmin_{f\in\calH}\tilde{\calL}^{(adv)}_{\lambda}(f),\,\,\,\text{where}\,\,\,\tilde{\calL}^{(adv)}_{\lambda}(f)=
\!\frac{1}{n}\!\sum_{i=1}^n\max_{\bdelta: \|\bdelta\|\leq\epsilon}(y_i\!-\!\langle f,K_{\bx_i}\!+\!\bT_{\bx_i}\bdelta\rangle)^2+\lambda\|f\|_{\calH}^2.
\end{align}

Similar to the proof of Theorem~\ref{thm:aug}, the proof of Theorem~\ref{thm:adv} will be based on two components: first, $\tilde{f}^{(adv)}_{\lambda}$ and $\hat{f}^{(adv)}_{\lambda}$ are close; second, $\tilde{f}^{(adv)}_{\lambda}$ and $g_{\lambda}$ are close. In fact, we will prove the equality that 
\[
\tilde{f}^{(adv)}_{\lambda}=\hat{f}_0+\Big(\bP_{\bK_{\bX}^\perp}^T\bSigma'\bP_{\bK_{\bX}^\perp}+\frac{\lambda}{\epsilon^2}\bI\Big)^{-1}\bP_{\bK_{\bX}^\perp}^T\bSigma\bP_{\bK_{\bX}}\hat{f}_0.
\] 
The detailed proof is rather technical and deferred to Section~\ref{sec:proof}.
\begin{remark}[Generalization to a generic set $\bDelta_0$]Following the proof of Theorem~\ref{thm:adv}, it can be generalized to $\bDelta_0$ that 
\begin{itemize}
\item is symmetric with respect to $0$
\item has the property that for any $\eta>0$, $\bx$ on the boundary of $\bDelta_0$, and $\bx'\in\bDelta_0$ such that $\|\bx'-\bx\|\geq \eta$, then $\|\bx\|^2-\bx^T\bx'>c\eta^2$. An exampletary set that satifies this condition is $\bDelta_0=\{\bDelta: \|\bDelta\|_p\leq 1\}$ for $1<p<\infty$.
\end{itemize}

Then $\|\hat{f}^{(aug)}_\lambda-\tilde{f}_\lambda\|_{\calH}\leq C(\epsilon^3/\lambda+\epsilon)$, where $\tilde{f}_\lambda$ can be defined as follows: $P_{\bK_{\bX}}\tilde{f}_\lambda=\hat{f}_0$ and $P_{\bK_{\bX}^\perp}\tilde{f}_\lambda$ is chosen by 
\[
P_{\bK_{\bX}^\perp}\tilde{f}_\lambda=\argmin_{h}\frac{1}{n}\sum_{i=1}^n\max_{\delta\in\bDelta}(\langle\hat{f}_0,P_{\bK_{\bX}}\bT_{\bx_i}\rangle+\langle h,P_{\bK_{\bX}^\perp}\bT_{\bx_i}\delta\rangle)^2+\lambda\|h\|_{\calH}^2.
\]
\end{remark}
\subsection{Discussion}
\textbf{Dependence on the order of $\epsilon$}
To understand Theorems~\ref{thm:aug} and~\ref{thm:adv}, an important setting is that $\lambda$ is in the order of $\epsilon^2$, or more generally, $O(\epsilon^{2.5}) < \lambda<o(\epsilon)$. Then both the augmented and the adversarial estimators are approximately by $g_{\lambda}$ or $g_{\lambda}'$ with  errors of $o(1)$. More specifically,

1. When $O(\epsilon^{2.5})<\lambda<o(\epsilon^2)$, $\hat{f}^{(aug)}$ and $\hat{f}^{(adv)}$ are approximated by $g_{0}$ and $g_0'$. 

2. When $O(\epsilon^{2})<\lambda<o(\epsilon)$, $\hat{f}^{(aug)}$ and $\hat{f}^{(adv)}$ are approximated by $\hat{f}_0$ and $\hat{f}_0'$.

3. When $\lambda$ is in the order of $\epsilon^2$, $\hat{f}^{(aug)}$ and $\hat{f}^{(adv)}$ can be viewed as an interpolation between $\hat{f}_0$ and $g_0$ or $g_0'$.

4. When $\lambda<o(\epsilon^2)$, then in the subspace $\Span(\bK_{\bX},\bT_{\bX})$, the augmented and the adversarial estimator are approximated by $g_{0}$ and $g_0'$.

\textbf{Conjecture for sufficiently small $\lambda$} Existing results show that when $\lambda<o(\epsilon^2)$, $\hat{f}^{(aug)}$ and $\hat{f}^{(adv)}$ are orthogonal to $\bT_{\bx_i}$ for all $1\leq i\leq n$, the coefficients of the first order expansion of $\bK_{\bx_i}$. Following the same idea, we conjecture that when $\lambda<o(\epsilon^{2k})$, $\hat{f}^{(aug)}$ and $\hat{f}^{(adv)}$ are orthogonal to the coefficients of the $k$-th order expansion of $\bK_{\bx_i}$. Therefore, when $k$ is large, $\hat{f}^{(aug)}$ and $\hat{f}^{(adv)}$ might be very different and has a larger functional norm in RKHS. 



\textbf{Other choice of perturbation set $\bDelta$} The similarity of the limiting formula between augmented and adversarial estimator in Theorems~\ref{thm:aug} and~\ref{thm:adv} are partially due to the perturbation set $\bDelta_0$ being a unit $\ell_2$ ball in Theorem \ref{thm:adv}. However, followng the discussion after Theorem \ref{thm:adv}, We expect that this equivalence would not hold for other sets, for example, when $\bDelta_0$ is an $\ell_p$ ball with $p\neq 2$, which corresponds to the $\ell_p$ attack in literature~\cite{10.5555/3495724.3496018}.

\section{Special cases}\label{sec:examples}
This section studies a few specific models to understand the Lipschitz constants and generalization errors of $\hat{f}^{(adv)}_\lambda$ and $\hat{f}^{(aug)}_\lambda$ and how they depend on $\lambda$. The generalization error is the standard measurement of the ``goodness of fit'', and we are also interested in the Lipschitz constant as it is a common measure of the robustness of neural network \cite{bubeck2021law}: if the fitted function has a large Lipschitz constant, then small adversarial attacks can change the prediction dramtically. 

Our study demonstrates that across different scenarios, the augmented and adversarial estimators exhibit higher generalization errors and larger Lipschitz constants compared to the standard estimator when $\lambda$ is very small, even when unregularized ($\lambda=0$). However, by selecting an appropriate value for $\lambda$, the regularized augmented and adversarial estimators can yield lower generalization errors and smaller Lipschitz constants. In summary, the underregularized augmented and adversarial estimators lead to overfitting and functions that are vulnerable to adversarial attacks. Nonetheless, regularization can address both issues.

As shown in Theorem~\ref{thm:aug} and~\ref{thm:adv}, the limiting properties of the augmented and adversarial estimators are very similar. As a result,  our results in this section hold for both estimators (with assumptions of $\bDelta_0$ in Theorem~\ref{thm:aug} and~\ref{thm:adv}) and to simplify notations, we use $\hat{f}^{(a)}_{\lambda}$ to represent both $\hat{f}^{(adv)}_\lambda$ and $\hat{f}^{(aug)}_\lambda$. In addition, we define the Lipschitz constant of $\hat{f}$ by
\[
\Lip(\hat{f})=\sup_{\bx\in\calS}\|\nabla\hat{f}(\bx)\|,
\]
where the set $\calS$ is the domain of the function, and the mean squared generalization error is defined by
\[
\MSE(\hat{f})=\Expect_{\bx\in\mu_{\calS}}\|\hat{f}(\bx)-f^*(\bx)\|^2,
\]
where $\mu_{\calS}$ is the distribution of $\bx$ in $\calS$ and $f^*$ is the ``true'' function.

\subsection{A two-point example} 
This section considers a very simple case where $n=2$ and the two points $\bx_1$ and $\bx_2$ are close, and investigates the performance of the estimated function over the line segment connecting $\bx_1$ and $\bx_2$.

The main result, summarized in Theorem~\ref{thm:simplecase}, implies that the generalization error and the Lipschitz constant of the ridgeless standard estimator $\hat{f}_0$ and  $\hat{f}_{\lambda}^{(a)}$ for various choices of $\lambda$ is ranked from good to bad as follows:
\begin{itemize}
\item With an appropriate regularization in the order of $\epsilon^2$, $\hat{f}_{\lambda}^{(a)}$ has the best performance.
\item With a regularization in the range of $(\epsilon^2,\epsilon)$, $\hat{f}_{\lambda}^{(a)}$ exhibits intermediate performance. The ridgeless estimator $\hat{f}_0$ also has a similar performance.
\item  Inadequate regularization in the order of $o(\epsilon^2)$, renders $\hat{f}_{\lambda}^{(a)}$ inferior to the estimators above. 
\end{itemize} 
It implies an overfitting phenomenon of $\hat{f}_{\lambda}^{(a)}$ when the regularization parameter is too small, and additional regularization can address the issue.


 \begin{thm}\label{thm:simplecase}[A two-point example]
Assume that $K_{\bX}$ has a Taylor expansion of the fourth order, and these taylor expansion coefficients are linearly independent, 
and let $\calS$ be the line segment connecting $\bx_1$ and $\bx_2$.
In addition, assume that $f^*$ is a linear function, i.e., $f^*(\bx_1+t\bu)=y_1+t(y_2-y_1)$ for $\bu=(\bx_2-\bx_1)/\|\bx_2-\bx_1\|$, and $\|\bx_1-\bx_2\|<r$ for some small constant $r>0$.  Then for a sufficiently small $\epsilon$, there exists $\lambda_2=O(\epsilon^2)$ and constants $C,c>0$ such that for all $\lambda_1\leq \epsilon^2/C$ and $C\epsilon^2\leq \lambda_3\leq \epsilon/C$, 
we have \begin{align}&\text{$0.9\Lip(\hat{f}^{(a)}_{\lambda_1})> \Lip(\hat{f}^{(a)}_{\lambda_3})\approx \Lip(\hat{f}_0) > (1+c)\Lip(\hat{f}^{(a)}_{\lambda_2})$ }
\label{eq:twopoint1}\\&\label{eq:twopoint2}
\text{$0.9\MSE(\hat{f}^{(a)}_{\lambda_1})> \MSE(\hat{f}^{(a)}_{\lambda_3})\approx \MSE(\hat{f}_0)> (1+c)\MSE(\hat{f}^{(a)}_{\lambda_2})$.}
\end{align}
Here $\approx$ represents a difference in the order of $o(1)$ as $\epsilon\rightarrow 0$.
\end{thm}
Note that $\lambda_1\leq \lambda_2\leq \lambda_3$, Theorem~\ref{thm:simplecase} implies that for both the ``large $\lambda$'' regime where $C\epsilon^2\leq \lambda_3\leq \epsilon/C$ and the ``small $\lambda$'' regime where $\lambda_1\leq \epsilon^2/C$, the performance is not as good as the ``intermediate $\lambda$'' regime $\lambda_2=O(\epsilon^2)$. In addition, the small $\lambda$ regime performs worse than the large $\lambda$ regime.


\begin{remark}[Overfitting in terms of the functional norm]
With a very small $\lambda$, the proof of Theorem~\ref{thm:simplecase} implies that $\|\hat{f}^{(aug)}_{\lambda_1}\|_{\calH}/\|\hat{f}^{(aug)}_{\lambda_2}\|_{\calH}$ and $\|\hat{f}^{(adv)}_{\lambda_1}\|_{\calH}/\|\hat{f}^{(adv)}_{\lambda_2}\|_{\calH}$  is in the order of $1/r^2$. That is, the augmented and adversarial estimators have much larger functional norms than the standard estimator. In some literature, a larger functional norm is considered as an indicator of overfitting \cite{pmlr-v80-belkin18a} as most of the available bounds of  generalization errors  depend on the RKHS norm.
\end{remark}


\subsection{Quadratic kernel}
This section considers the quadratic kernel of $K(\bx,\by)=a_1^2\bx^T\by+a_2^2(\bx^T\by)^2$. We remark that the coefficients of $\bx^T\by$ and $(\bx^T\by)^2$ are nonnegative, so that the associated kernel $K$ is positive semidefinite: for any data matrix $\bX\in\reals^{n\times p}$, the associated $n\times n$ kernel matrix is $a_1^2\bX\bX^T+a_2^2(\bX\bX^T)\circ(\bX\bX^T)$, where $\circ$ represents the Hadamard product of two matrices. Then by Schur product theorem  and the fact that $\bX\bX^T\in\reals^{n\times n}$ is positive semidefinite, the kernel matrix is also positive semidefinite. 

With this choice of kernel, we have the following statement that is very similar to Theorem~\ref{thm:simplecase}. Again,  with a very small $\lambda$ (including the unregularized case $\lambda=0$), the augmented and adversarial estimators are not as robust as the standard estimator (measured by $\Lip$) and have larger generalization errors; but with appropriately chosen $\lambda$, the regularized augmented and adversarial estimators are more robust than the standard estimator and have smaller generalization errors. 
\begin{thm}\label{thm:quadratic}[Quadratic kernel]
If $\bx_i=\be_i$ for $1\leq i\leq n$ and $n\leq p$, and  $y_i$ are i.i.d. sampled from a distribution on $\reals$, and $\calS$ be the unit ball in $\reals^p$. 
In addition, $f^*$ is a linear function, i.e., $f^*(\bx)=\sum_{i=1}^p y_ix_i$.
Then for a sufficiently large $n$ and small $\epsilon$, there exists $\lambda_2=O(\epsilon^2)$ and $C>0$ such that for all $\lambda_1\leq \epsilon^2/C$ and $C\epsilon^2\leq \lambda_3\leq \epsilon/C$, as $n\rightarrow\infty$,
we have \begin{align}\label{eq:lip_quadratic}&\text{$0.5\Lip(\hat{f}^{(a)}_{\lambda_1})> \Lip(\hat{f}^{(a)}_{\lambda_2})\approx \Lip(\hat{f}_0) > \Lip(\hat{f}^{(a)}_{\lambda_3})$ w.p. $1-6.33 \times 10^{-5}$,}
\\&
(1-c)\text{$\MSE(\hat{f}^{(a)}_{\lambda_1})> \MSE(\hat{f}^{(a)}_{\lambda_2})\approx \MSE(\hat{f}_0)> \MSE(\hat{f}^{(a)}_{\lambda_3})$ almost surely.}\label{eq:mse_quadratic}
\end{align} 


\end{thm}

We would like to point out that the assumption $\bx_i=\mathbf{e}_i$ can be interpreted as an approximation of the scenario where $\bx_i$ are independent and identically distributed samples from a multivariate normal distribution with mean zero and covariance matrix $\bI_{p}$, where $n$ is much smaller than $p$. This is because standard measure concentration results suggest that as $p$ increases to infinity, any two random vectors drawn from $N(0,\bI_{p})$ become approximately orthogonal to each other and have comparable magnitudes.
 
\begin{remark}\label{remark:quadratic}
If $y_i$ are i.i.d. sampled from a distribution on $\reals$ with nonzero mean, then \eqref{eq:lip_quadratic} holds with probability $1$. In addition, $\Lip(\hat{f}^{(a)}_{\lambda_1})/\Lip(\hat{f}^{(a)}_{\lambda_2})$ and $\|\hat{f}^{(a)}_{\lambda_1}\|_{\calH}/\|\hat{f}^{(a)}_{\lambda_2}\|_{\calH}$ are both in the order of $O(\sqrt{n})$, which means that for small $\lambda$, the augmented and adversarial estimators  have large Lipschitz constants and large functional norms. 
\end{remark}


\subsection{Generic kernel}
In this section, we explore a general setting and demonstrate that when the regularization parameter $\lambda$ is small, augmented and adversarial estimators exhibit less robustness in the sense that their Lipschitz constants are larger than that of the standard estimator $\hat{f}_0$.



\begin{thm}[Generic kernel]\label{thm:general}
Assume that $\bT_{\bx}$ is continuous and $\bx_1,\cdots,\bx_n$ are i.i.d. uniformly sampled from a distribution $\mu$ with compact support $\mathcal{S}\subset\reals^p$, and $y_i=f^*(\bx_i)$ for some $f^*\in\calH$. 
Then for a sufficiently large $n$, there exists $\epsilon_0$ and $C>0$ such that for all $\epsilon\leq \epsilon_0$, $\lambda_1\leq \epsilon^2/C$, and $C\epsilon^2\leq \lambda_3\leq \epsilon/C$, \[0.9 \Lip(\hat{f}^{(a)}_{\lambda_1})> \Lip(\hat{f}^{(a)}_{\lambda_3})\approx \Lip(\hat{f}_0).\] 
\end{thm}

It should be noted that although we are unable to provide a formal proof demonstrating that augmented and adversarial estimators outperform the standard estimation with appropriate regularization as in Theorems~\ref{thm:simplecase} and~\ref{thm:quadratic}, we do observe this phenomenon empirically in Section~\ref{sec:simulations}.

\begin{remark}[Overfitting in terms of the functional norm]\label{remark:functionalnorm_general}
For some settings, we may still show that $\|\hat{f}^{(aug)}_{\lambda_1}\|/\|\hat{f}^{(aug)}_{\lambda_2}\|$ is large, which implies that the ``very small $\lambda$'' regime would not generalize well. For example, when $\bx_1,\cdots,\bx_n$ are sampled uniformly from $[0,1]$, assuming that $f^*=K_{\bx_0}$ for some $\bx_0\in [0,1]$, then 
\[
\frac{\|\hat{f}^{(aug)}_{\lambda_1}\|_{\calH}}{\|\hat{f}^{(aug)}_{\lambda_2}\|_{\calH}}\geq \frac{K(\bx_0,\bx_0)}{\Expect_{\bx,\by\in[0,1]}K(\bx,\by)+O(1/n)},
\]
and for $K(\bx,\by)=\exp(-\|\bx-\by\|^2/\sigma^2)$, the RHS is in the order of $1/(\min(\sigma,1)+O(1/n))$, which goes to $\infty$ as $\sigma\rightarrow 0$ and $n\rightarrow\infty$. 
\end{remark}

\section{Simulations}\label{sec:simulations}
In this section, we conduct simulations on synthetic datasets as well as the MNIST~\cite{lecun-mnisthandwrittendigit-2010} dataset, with the primary objective of substantiating the findings presented in Section~\ref{sec:examples}. Specifically, we aim to demonstrate that for augmented and adversarial estimators, minimal regularization leads to increased generalization errors and Lipschitz constants.  We compare the performance of various estimator with a range of regularization parameters using the out-of-sample mean squared error metric on the test dataset
\[
\overline{\MSE}(\hat{f})=\mathrm{Average}_{\text{test data}\,\, (\bx,y)}(\hat{f}(\bx)-y)^2
\]
and the estimated Lipschitz constant of $\hat{f}$ based on the test dataset:
\[
\overline{\Lip}(\hat{f})=\max_{\text{test data} \,\,(\bx,y)} \|\grad_{\bx}\hat{f}(\bx)\|.
\]
\textbf{Augmented estimator}  
We employ a variant of \eqref{eq:augment} in which the augmentation set varies for each data point:
\begin{equation}\label{eq:augment1}
\hat{f}_{\lambda}^{(aug)}=\argmin_{f\in\calH}\frac{1}{n}\sum_{i=1}^n\Expect_{\bdelta\sim \bDelta_i}(y_i-f(\bx_i+\bdelta))^2+\lambda\|f\|_{\calH}^2,
\end{equation}
where $\bDelta_i$ consists of $K$ random unit vectors of length $\epsilon$. This can be considered as the standard procedure of ``adding noise to input data''.

We commence by generating synthetic datasets as follows: for each sample $(\bx,y)\in\reals^p\times \reals$, $\bx$ is a random unit vector in $\reals^p$ and $y=x_1+x_2^2$. In particular, we consider two settings:
\begin{enumerate}
  \item  $p=3$, $50$ training samples, $500$ testing samples, with Gaussian kernel $K(\bx,\bx')=\exp(-10\|\bx-\bx'\|^2)$.

\item $p=20$, $150$ training samples, $500$ testing samples, with quadratic kernel $K(\bx,\bx')=\bx^T\bx'+10(\bx^T\bx')^2$.
\end{enumerate}
For both simulations, we use $K=40$ in data augmentation. We repeat the simulation 100 times, and the average estimated MSE and Lipschitz constants of the standard estimator \eqref{eq:traditional} and the augmented estimator \eqref{eq:augment} over a grid of regularization parameters are presented in Figure~\ref{fig:simulation}. It shows that the generalization error and the Lipschitz constant of augmented estimators decrease as $\lambda$ increases from $0$; and after an `optimal' point in the order of $\epsilon^2$ it would increase. (We remark that the small Lipschitz constant for large $\lambda$ is due to the fact that the $\hat{f}$ converges to $0$ and therefore its Lipschitz constant converges to $0$.) This implies that without sufficient regularization, augmented estimators have overfitting problems and produce  functions that are not robust, but both problems problem can be mitigated using regularization. In fact, augmented estimators with appropriate  regularization perform better than the standard estimator, which performs  better than augmented estimators with insufficient regularization. This verifies our theoretical investigation in Section~\ref{sec:examples}. We also note that as a comparison, the standard regression does not suffer from overfitting.

\begin{figure}\label{fig:simulation}
\includegraphics[width=0.45\textwidth]{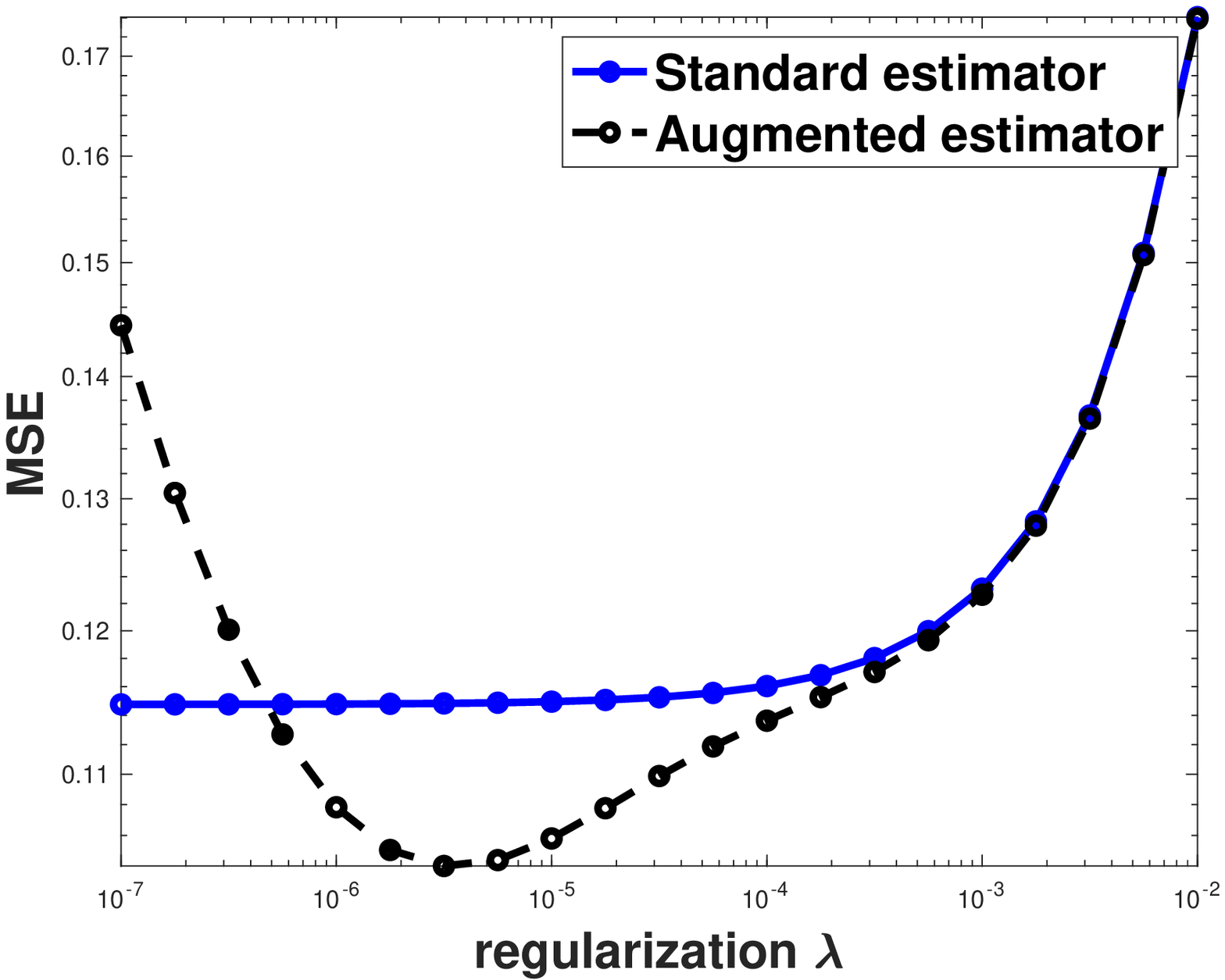}
\includegraphics[width=0.45\textwidth]{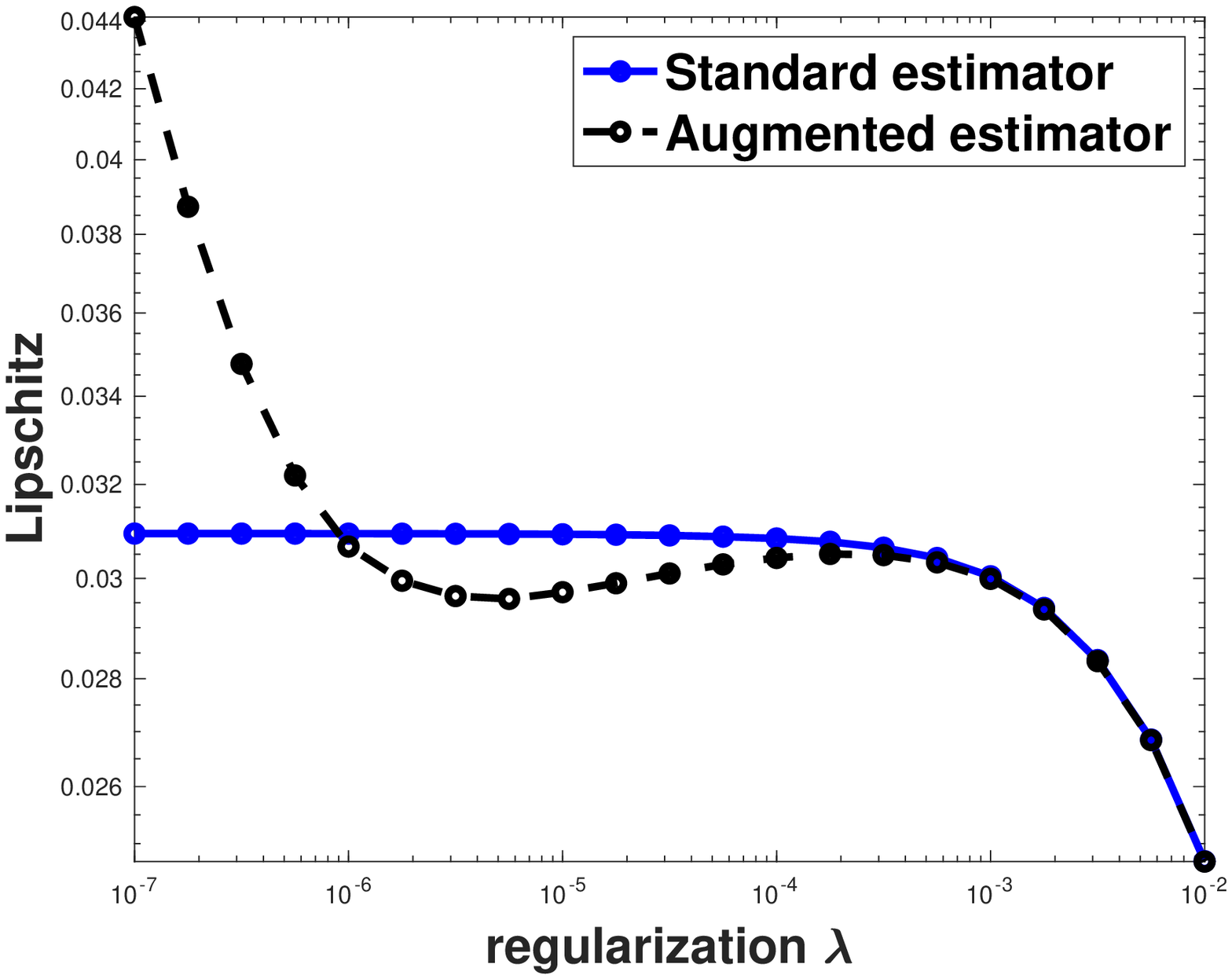}\\
\includegraphics[width=0.45\textwidth]{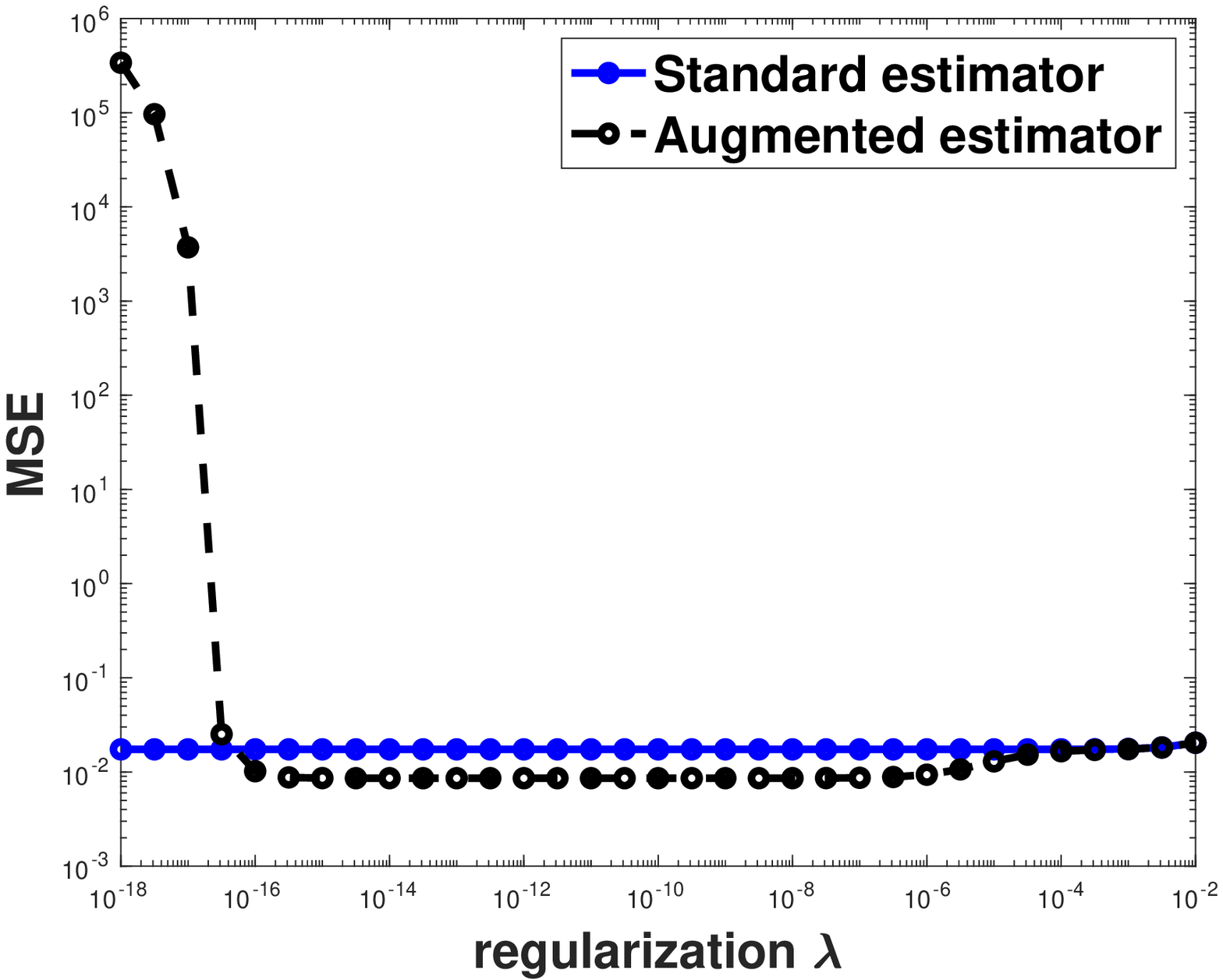}
\includegraphics[width=0.45\textwidth]{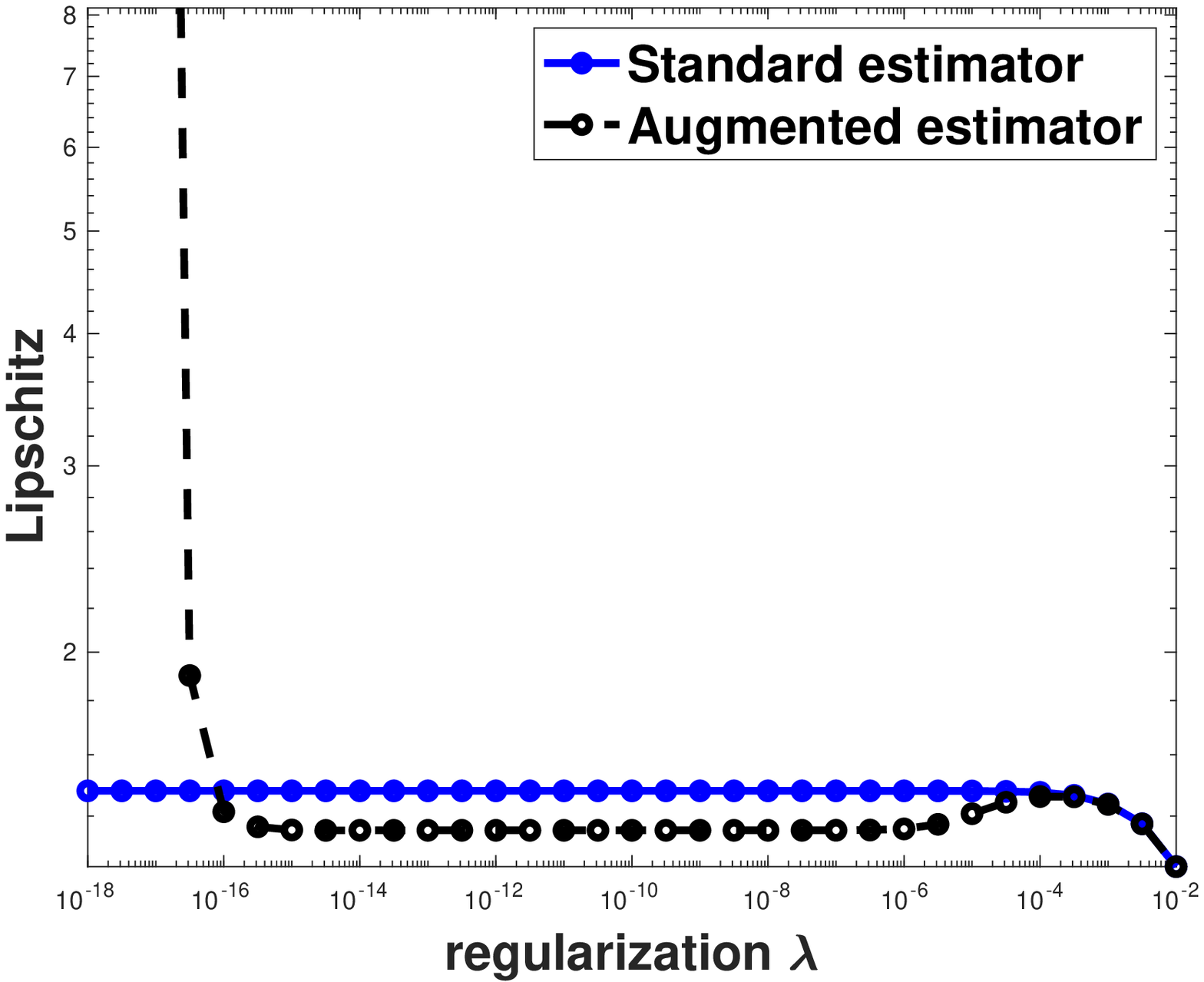}
\caption{The recorded $\overline{\Lip}$ and $\overline{\MSE}$ for augmented and standard estimators in simulated data sets for various choices of the regularization parameter. The first row represents setting 1, and the second row represents setting 2.}
\end{figure}

\begin{table} \centering
\begin{tabular}{l*{8}{c}}
regularization               & $10^{-14}$ & $10^{-13}$ & $10^{-12}$ & $10^{-11}$  & $10^{-10}$ & $10^{-9}$& $10^{-8}$ & $10^{-7}$  \\
\hline
\hline
  \multicolumn{9}{c}{(2,7)} \\
\hline
$\overline{\MSE}_{\text{augment}}$& 0.1615 &   0.1614   & \textbf{0.1603} &   0.1610   & 0.1831    &0.2080&0.2141 &   0.2148  \\
$\overline{\MSE}_{\text{standard}}$& 0.2149  &  0.2149   & 0.2149 &   0.2149 &   0.2149  &  0.2149  &  0.2149  &  0.2149\\
$\overline{\Lip}_{\text{augment}}$&3.3203  &  \textbf{3.3198}    &3.3284   & 3.5998  &  4.6664  &  5.5526 &   5.7744 &   5.8003\\
$\overline{\Lip}_{\text{standard}}$&  5.8059  &  5.8059  &  5.8059   & 5.8059  &  5.8059  &  5.8058  &  5.8056   & 5.8035 \\

\hline
\hline
 \multicolumn{9}{c}{(3,4)} \\
\hline
$\overline{\MSE}_{\text{augment}}$&  0.0968   & 0.0967 &   \textbf{0.0965}  &  0.0993  &  0.1207 &   0.1432 &   0.1486   & 0.1493    \\
$\overline{\MSE}_{\text{standard}}$& 0.1494  &  0.1494  &  0.1494  &  0.1494  &  0.1494  &  0.1494   & 0.1494 &   0.1493  \\
$\overline{\Lip}_{\text{augment}}$&\textbf{3.1697}  &  3.1699 &   3.1806 &   3.4031  &  4.2551 &   4.9986 &   5.1918 &   5.2144\\
$\overline{\Lip}_{\text{standard}}$& 5.2193   & 5.2193 &   5.2193  &  5.2193 &   5.2193 &   5.2192 &   5.2191   & 5.2172 \\

\hline
\hline
 \multicolumn{9}{c}{(1,6)} \\
\hline
$\overline{\MSE}_{\text{augment}}$&  0.0685  &  0.0680   & 0.0646   & \textbf{0.0607}    &0.0720   & 0.0852 &   0.0884&0.0887 \\
$\overline{\MSE}_{\text{standard}}$&0.0888   & 0.0888  &  0.0888 &   0.0888 &    0.0888   & 0.0888 &   0.0888 &   0.0888\\
$\overline{\Lip}_{\text{augment}}$&3.6859   & 3.6601 &   3.5097&    \textbf{3.4695}   & 4.2919 &   4.9662 &   5.1366 & 5.1533\\
$\overline{\Lip}_{\text{standard}}$&5.1616 &   5.1616 &   5.1616  &  5.1616    &5.1616  &  5.1616 &   5.1610 & 5.1558\\

 \end{tabular}\label{tab:mnist}
 \caption{Experiments on MNIST data set for three pairs of digits: (2,7), (3,4), (1,6). The recorded $\overline{\Lip}$ and $\overline{\MSE}$ for augmented and standard estimators for various choices of the regularization parameter $\lambda$ are shown, and the smallest MSE and Lipschitz constants are marked in bold.}
\end{table}

Our second data set is sampled from MNIST~\cite{lecun-mnisthandwrittendigit-2010}. For each pair of distinct digits $(i,j)$, where $i,j\in\{0,1,\cdots,9\},$ label one digit as $1$ and the other as $-1$. For each simulation, the training set has $100$ samples and the testing set has $2000$ samples. We follow \cite{10.1214/19-AOS1849} and use  the Gaussian kernel 
$k(\bx, \bx') = \exp(-\|\bx - \bx'\|^2/768)$ and $K=40$ in data augmentation. We remark that the training samples are rather small to ensure that the augmented estimator can be computed as its size is $K$ times larger. The results are recorded in Table~\ref{tab:mnist}, in which a similar phenomenon as in Figure~\ref{fig:simulation} have been observed: with small regularization parameters, the augmented estimators does not perform as well as the augmented estimators with appropriate regularizations. However, it is slightly different than Figure~\ref{fig:simulation} and our findings in Section~\ref{sec:examples} in the sense that even with a very small $\lambda$, the augmented estimator $\hat{f}^{(aug)}_{\lambda}$ still outperforms the standard estimator $\hat{f}_0$. We suspect that it might be due to numerical issues when some singular values of $\bK_{\bX}$ or $\bT_{\bX}$ are small.


\textbf{Adversarial estimator}  We perform simulations on the adversarial estimator \eqref{eq:adversarial} separately as there is no explicit solution and the commonly used gradient descent algorithm \cite{gao2019convergence} for adversarial training converges slowly. Instead, we record the $\overline{\MSE}$ and $\overline{\Lip}$ for the estimated functions over the iterations gradient descent, since early stopping in a ridgeless regression can be viewed as a form of regularization. 
For the simplicity of calculation, we apply a variant of \eqref{eq:adversarial}, following the idea similar to \eqref{eq:augment1}:
\begin{align}
\hat{f}_{\lambda}^{(adv)}&=\argmin_{f\in\calH}\frac{1}{n}\sum_{i=1}^n\max_{\bdelta\in\bDelta_i}(y_i-f(\bx+\bdelta))^2+\lambda\|f\|_{\calH}^2,
\end{align}
where $\bDelta_i$ is generated using the method after \eqref{eq:augment1}. Here we use the simulated data set of  setting 1 and let the step size at the $k$-th iteration be $\min(0.04, 4/k^{0.5})$. The performance of adversarial training over iteration is recorded in Table~\ref{tab:adv} and we use the average over 50 repeats. It verifies the findings in Section~\ref{sec:examples} that adversarial training overfits in terms of both generalization error and Lipschitz constant, and regulatizations mitigates the overfitting issue. 
\begin{table} \centering
\begin{tabular}{l*{8}{c}}
iterations             & $10$ & $30$ & $100$ & $300$  & $1000$ & $3000$& $10000$ & $30000$  \\
\hline
MSE &  0.2378  &  0.2353 &   0.1555   & 0.1422  &  \textbf{0.1357}   & 0.1377   & 0.1444   & 0.1511\\
Lipschitz &      0.0462  &  0.0397  &  0.0294  &  0.0293 &   \textbf{0.0285}  &  0.0286  &  0.0323  &  0.0353
 \end{tabular}\label{tab:adv}
 \caption{The recorded $\overline{\Lip}$ and $\overline{\MSE}$ of ridgeless adversarial training over iterations on the simulated data set described in setting 1.}
\end{table}

\section{Conclusion}
This paper presents a study of the performance of adversarial and augmented kernel regression compared to the ridgeless standard regression, and derives limiting formulas for these methods. The analysis reveals that when the regularization parameter is small, augmented or adversarial regression may lead to unstable performance, with worse generalization error and Lipschitz constant than ridgeless standard regression. However, with appropriate regularization, the regularized augmented or adversarial estimators outperform the ridgeless standard regression, demonstrating the effectiveness of regularization in adversarial training and data augmentation with noise, which differs from the standard kernel regression setting where overfitting is generally not a concern.

In future work, it would be interesting to explore the impact of $\epsilon$ that does not depend on $n$ and $\bX$, which might be more realistic in settings with very large values of $n$. Additionally, it would be valuable to investigate strategies for choosing an optimal value of $\lambda$ or the best early stopping criterion based on the kernel and data sets.

\section{Proof}\label{sec:proof}
\begin{proof}[Proof of Theorem~\ref{thm:aug}]
(a) To prove Theorem~\ref{thm:aug}, we first establish the following lemma:
\begin{lemma}[Perturbation of a linear system]\label{lemma:linearalgebra}
(a) Assuming that $\bA,\bB\in\reals^{m\times p}, \by\in\reals^m$, $rank(\bA)=n<m$, $\by\in\Span(\bA)$, where $\Span(\bA)$ represents the column space of $\bA$, $\bC=\bA+\epsilon\bB$, and let
\[
\tilde{\bx}=\argmin_{\bx\in\reals^p} \|\bC\bx-\by\|^2+\lambda\|\bx\|^2, \hat{\bx}=\argmin_{\bx\in\reals^p} \|\bA\bx-\by\|^2+\lambda\|\bx\|^2,
\]
then
\[
\|\tilde{\bx}-\hat{\bx}\|=O\Big(\frac{\epsilon^2/\lambda+\epsilon}{\sigma_{n}(\bA)}\Big).
\]
(b) If in addition $\bA\bB^T=0$, $\sigma_n(\bA)=O(1)$, and let $L\in\reals^p$ be the row space of $\bA$, then 
\[
\Big\|\tilde{\bx}-\hat{\bx}+\Big(\bP_{\bL^\perp}^T\bB\bB^T\bP_{\bL^\perp}+\frac{\lambda}{\epsilon^2}\bI\Big)^{-1}\bP_{\bL^\perp}^T\bB\bB^T\bP_{\bL}\hat{\bx}\Big\|=O(\epsilon^4/\lambda+\epsilon^2).
\]
\end{lemma}
The proof of  Theorem~\ref{thm:aug}(a) is divided into two components. First, we will show that $\tilde{f}^{(aug)}_{\lambda}$ and $\hat{f}^{(aug)}_{\lambda}$ are close; second, $\tilde{f}^{(aug)}_{\lambda}$ and $g_{\lambda}$ are close. Both components depend on Lemma~\ref{lemma:linearalgebra}.

For the first component, note that $\hat{f}^{(aug)}_{\lambda}$ and $\tilde{f}^{(aug)}_{\lambda}$ can be considered as linear systems with $n\times |\bDelta|$ measure vectors (here $|\bDelta|$ can be infinity), represented by $K_{\bx_i+\bdelta}$ and $K_{\bx_i}+\bT_{\bx_i}\bDelta$ respectively. In addition, $K_{\bx_i+\bdelta}-(K_{\bx_i}+\bT_{\bx_i}\bDelta)=O(\epsilon^2)$. Then, Lemma~\ref{lemma:linearalgebra}(a) with $\bA$ being a matrix with row vectors $K_{\bx_i}+\bT_{\bx_i}\bDelta$, $\bC$ being a matrix with row vectors $K_{\bx_i+\bdelta}$, $y_i$ being the corresponding measurement of $K_{\bx_i+\bdelta}$, $\epsilon$ replaced with $\epsilon^2$, $n$ replaced with $n(p+1)$. By the assumption, we have the $n(p+1)$-th singular values of the matrix with row vectors $K_{\bx_i}+\bT_{\bx_i}\bDelta$ is in the order of $\Omega(\epsilon)$, and it can be verified that $\by\in\Span(\bA)$ (note that by solving the system of $K_{\bx_i}\bz=y_i$ and $\bT_{\bx_i}\bz=0$), so we have 
\begin{equation}\label{eq:aug_parta}
\Big\|\tilde{f}^{(aug)}_\lambda-\hat{f}_{\lambda}+\Big(\bP_{\bK_{\bX}^\perp}^T\bB^T\bB\bP_{\bK_{\bX}^\perp}+\frac{\lambda}{\epsilon^2}\bI\Big)^{-1}\bP_{\bK_{\bX}^\perp}^T\bB^T\bB\bP_{\bK_{\bX}}\hat{f}_{\lambda}\Big\|=O\Big(\frac{\epsilon^4/\lambda+\epsilon^2}{\epsilon}\Big)=O\Big(\epsilon^3/\lambda+\epsilon\Big).
\end{equation}

For the second part, applying Lemma~\ref{lemma:linearalgebra}(b) with $\bA$ being a matrix with row vectors $K_{\bx_i}$ and $\bC$ being a matrix with  row vectors $K_{\bx_i}+\bT_{\bx_i}\bDelta$. Since $\Expect\bDelta=0$, we have $\bA\bB^T=\bA\bC^T-\bA\bA^T=0$. In addition, the rank of $\bA$ is $n$ and $\sigma_n(\bA)=\Omega(1)$ by assumption, and $\bC-\bA=O(\epsilon)$, so Lemma~\ref{lemma:linearalgebra}(b) implies that
\begin{equation}\label{eq:aug_partb}
\Big\|\tilde{f}^{(aug)}_{\lambda}-\hat{f}_{\lambda}+\Big(\bP_{\bL^\perp}^T\bSigma\bP_{\bL^\perp}+\frac{\lambda}{\epsilon^2}\bI\Big)^{-1}\bP_{\bL^\perp}^T\bSigma\bP_{\bL}\hat{f}_{\lambda}\Big\|=O(\epsilon^4/\lambda+\epsilon^2).
\end{equation}
Combining \eqref{eq:aug_parta} and \eqref{eq:aug_partb}, \eqref{eq:aug1} is proved. 

(b) Note that 
\[
\hat{\calL}^{(aug)}(g_0)=\tilde{\calL}^{(aug)}(g_0)+(\hat{\calL}^{(aug)}(g_0)-\tilde{\calL}^{(aug)}(g_0))+\lambda\|g_0\|^2=0+O(\epsilon^4+\lambda)=O(\epsilon^4+\lambda),
\]
and
\[
\hat{\calL}^{(aug)}(f)\geq C\epsilon^{2}\Big\|P_{\Span(\bK_{\bX},\bT_{\bX})}f_\lambda-g_0 \Big\|^2,
\]
Theorem~\ref{thm:aug}(b) is proved. 

\end{proof}

\begin{proof}[Proof of Lemma~\ref{lemma:linearalgebra}]
(a) Since $\by\in\Span(\bA)$, WLOG we may rearrange the space $\reals^m$ and assume that 
\[
\bA=\begin{pmatrix}\bA_{1} \\
\mathbf{0}_{(m-n)\times p}
\end{pmatrix},\,\,\by=\begin{pmatrix}\bx_{1} \\
\mathbf{0}_{m-n}
\end{pmatrix},\,\,\, \bB=\begin{pmatrix}\bB_{1}\\
\bB_{2}
\end{pmatrix}, 
\]
for some $\bA_1,\bB_1\in\reals^{n\times p}$, $\bx_1\in\reals^n$, and $\bB_2\in\reals^{m-n\times p}$.

Then we have
\begin{align}\nonumber
&\tilde{\bx}=\bC^T(\bC\bC^T+\lambda\bI)^{-1}\begin{pmatrix}\bx_{1} \\
\mathbf{0}_{m-n}
\end{pmatrix}\\=&\begin{pmatrix}\bA_1+\epsilon\bB_1\\
\bB_{2}
\end{pmatrix}^T
\begin{pmatrix}
(\bA_1+\epsilon\bB_1)(\bA_1+\epsilon\bB_1)^T+\lambda\bI & (\bA_1+\epsilon\bB_1)\epsilon\bB_2^T \\\nonumber
\epsilon\bB_2(\bA_1+\epsilon\bB_1)^T & \epsilon^2\bB_2\bB_2^T +\lambda\bI
\end{pmatrix}^{-1}\begin{pmatrix}\bx_{1} \\
\mathbf{0}_{m-n}
\end{pmatrix}\\\nonumber
=&\begin{pmatrix}\bA_1+\epsilon\bB_1\\
\bB_{2}
\end{pmatrix}^T \begin{pmatrix}
\bZ &  *\\
-(\epsilon^2\bB_2\bB_2^T +\lambda\bI)^{-1}(\epsilon\bB_2(\bA_1+\epsilon\bB_1)^T)\bZ & *
\end{pmatrix}\begin{pmatrix}
\bx_1 \\
0
\end{pmatrix}\nonumber
\\\label{eq:lemma_step1}
=&\Big(\bI-\epsilon\bB_2^T (\epsilon^2\bB_2\bB_2^T +\lambda\bI)^{-1}\epsilon\bB_2\Big)(\bA_1+\epsilon\bB_1)^T\bZ\bx_1.
%
\end{align}
where the expression of $\bZ$ follows from the well-known $2\times 2$ block matrix inversion formula
\begin{equation}\label{eq:inversion}
\begin{pmatrix}
\bA & \bU \\
\bV& \bC
\end{pmatrix}^{-1}
=\begin{pmatrix}
(\bA-\bU\bC^{-1}\bV)^{-1} & -(\bA-\bU\bC^{-1}\bV)^{-1}\bU\bC^{-1} \\
-\bC^{-1}\bV(\bA-\bU\bC^{-1}\bV)^{-1}& \bC^{-1}+\bC^{-1}\bV(\bA-\bU\bC^{-1}\bV)^{-1}\bU\bC^{-1}
\end{pmatrix}
\end{equation}
and we have
\begin{align*}\bZ^{-1}=&(\bA_1+\epsilon\bB_1)(\bA_1+\epsilon\bB_1)^T+\lambda\bI-(\bA_1+\epsilon\bB_1)\epsilon\bB_2 ^T(\epsilon^2\bB_2\bB_2^T +\lambda\bI)^{-1}\epsilon\bB_2(\bA_1+\epsilon\bB_1)^T \\
=&(\bA_1+\epsilon\bB_1)\bQ(\bA_1+\epsilon\bB_1)^T+\lambda\bI\end{align*}
for the symmetric matrix $\bQ=\bI-\epsilon\bB_2^T (\epsilon^2\bB_2\bB_2^T +\lambda\bI)^{-1}\epsilon\bB_2=\bI+O(\epsilon^2/\lambda)$. Then 
\begin{equation}
(\bA_1+\epsilon\bB_1)^T\bZ=(\bA_1+\epsilon\bB_1)^T\Big((\bA_1+\epsilon\bB_1)\bQ(\bA_1+\epsilon\bB_1)^T+\lambda\bI\Big)^{-1}=\bQ^{-1/2}f(\bQ^{1/2}(\bA_1+\epsilon\bB_1))\label{eq:lemma_step2}
\end{equation}
for $f(\bX)=\bX^T(\bX\bX^T+\lambda\bI)^{-1}.$ Since $f(\bX)$ is continuous at $\bX$ with locally Lipschitz constant $1/\sigma_{\min}(\bX)$ \textbf{(add citations)} and $\bQ^{1/2}(\bA_1+\epsilon\bB_1)-\bA_1=O(\epsilon^2/\lambda+\epsilon)$, \eqref{eq:lemma_step2} implies that if $\epsilon\|\bB_1\|<\sigma_{\min}(\bA)/2$, then
\begin{equation}
(\bA_1+\epsilon\bB_1)^T\bZ=f(\bA_1)+O\Big(\frac{\epsilon^2/\lambda+\epsilon}{\sigma_{\min}(\bA_1)}\Big)=\bA_1^T(\bA_1\bA_1^T+\lambda\bI)^{-1}+O\Big(\frac{\epsilon^2/\lambda+\epsilon}{\sigma_{\min}(\bA_1)}\Big).
\label{eq:lemma_step3}\end{equation}
Combining \eqref{eq:lemma_step1} and \eqref{eq:lemma_step3} and $\epsilon\bB_2^T (\epsilon^2\bB_2\bB_2^T +\lambda\bI)^{-1}\epsilon\bB_2=O(\epsilon^2/\lambda)$, $\hat{\bx}=\bA^T(\bA\bA^T+\lambda\bI)^{-1}\by=\bA_1^T(\bA_1\bA_1^T+\lambda\bI)^{-1}\bx_1$, as well as ${\sigma_{\min}(\bA)}={\sigma_{\min}(\bA_1)}$ we have that if $\epsilon\|\bB\|<\sigma_{\min}(\bA)/2$, then
\[
\|\tilde{\bx}-\hat{\bx}\|=O\Big(\frac{\epsilon^2/\lambda+\epsilon}{\sigma_{\min}(\bA)}\Big).
\]

(b) 
If the column spaces of $\bA$ and $\bB$ are perpendicular, then we may rearrange the row space in $\reals^p$ such that
\[
\bA=\begin{pmatrix}\bA_{11} , 0\\
0,0
\end{pmatrix},\,\,\, \bB=\begin{pmatrix} 0, 0\\
\bB_{21},\bB_{22}
\end{pmatrix}
\]
for $\bA_{11}\in\reals^{n\times n}$, $\bB_{21}\in\reals^{(m-n)\times n}$, and $\bB_{22}\in\reals^{(m-n)\times (p-n)}$. Note that there is a connection that $
\bA_1=\begin{pmatrix}\bA_{11},
0
\end{pmatrix}, \bB_1=0, \bB_2=\begin{pmatrix} \bB_{21}, 
\bB_{22}
\end{pmatrix}.$

Then we have $\bC^T\bC=\bA^T\bA+\epsilon^2\bB^T\bB$, and
\begin{align*}
&\hat{\bx}=(\bC^T\bC+\lambda\bI)^{-1}\bC^T\begin{pmatrix}\bx_{1} \\
\mathbf{0}
\end{pmatrix}
\\=&\begin{pmatrix}
\bA_{11}^T\bA_{11}+\epsilon^2\bB_{21}^T\bB_{21}+\lambda\bI & \epsilon^2\bB_{21}^T\bB_{22} \\
\epsilon^2\bB_{22}^T\bB_{21}& \epsilon^2\bB_{22}^T\bB_{22} +\lambda\bI
\end{pmatrix}^{-1}\begin{pmatrix} \bA_{11}\bx_1\\
0
\end{pmatrix}
\\
=&\begin{pmatrix}
\bZ & *\\
-(\epsilon^2\bB_{22}^T\bB_{22} +\lambda\bI)^{-1}\epsilon^2\bB_{22}^T\bB_{21} \bZ  & * 
\end{pmatrix}\begin{pmatrix} \bA_{11}\bx_1\\
0
\end{pmatrix}
\end{align*}
where
\begin{align*}
\bZ=&\Big(\bA_{11}^T\bA_{11}+\epsilon^2\bB_{21}^T\bB_{21}+\lambda\bI - \epsilon^2\bB_{21}^T\bB_{22} (\epsilon^2\bB_{22}^T\bB_{22} +\lambda\bI)^{-1}\epsilon^2\bB_{22}^T\bB_{21} \Big)^{-1}\\
=&(\bA_1\bA_1^T+\lambda\bI)^{-1}+O(\epsilon^4/\lambda+\epsilon^2)
\end{align*}
follows from the inversion formula \eqref{eq:inversion} and the fact that the minimum singular value of $\bA_{11}$ is $O(1)$. Note that here we applies the assumption that
\[
\Big\|\epsilon^2\bB_{21}^T\bB_{21}+ \epsilon^4\bB_{21}^T\bB_{22}\bB_{22}^T\bB_{21}/\lambda \Big\|\leq \sigma_{\min}(\bA_1\bA_1^T)/2.
\]

Part (b) then follows from $\bZ\bA_{11}\bx_1=\hat{\bx}+O(\epsilon^4/\lambda+\epsilon^2)$ and $(\epsilon^2\bB_{22}^T\bB_{22} +\lambda\bI)^{-1}\epsilon^2\bB_{22}^T\bB_{21} \bZ=\Big(\bP_{\bL^\perp}^T\bB^T\bB\bP_{\bL^\perp}+\frac{\lambda}{\epsilon^2}\bI\Big)^{-1}\bP_{\bL^\perp}^T\bB^T\bB\bP_{\bL}\hat{\bx}+O(\epsilon^4/\lambda+\epsilon^2).$


\end{proof}

\begin{proof}[Proof of Theorem~\ref{thm:adv}(a)]
First, we summarize our findings in the following statement on the equivalence between $\tilde{f}^{(adv)}_{\lambda}$ and $g_{\lambda}'=\hat{f}_0+\Big(\bP_{\bK_{\bX}^\perp}^T\bSigma\bP_{\bK_{\bX}^\perp}+\frac{\lambda}{\epsilon^2}\bI\Big)^{-1}\bP_{\bK_{\bX}^\perp}^T\bSigma\bP_{\bK_{\bX}}\hat{f}_0$:
\begin{equation}\label{eq:adv_equivalency}
\text{
When $\lambda=o(\epsilon)$,  $\tilde{f}^{(adv)}_{\lambda}=g_{\lambda}'$.}
\end{equation}

To prove \eqref{eq:adv_equivalency}, we only need to show that the derivative of $\tilde{\calL}^{(adv)}_{\lambda}$ at $g_{\lambda}'$ is zero. It is straightforward to verify that the objective function in \eqref{eq:auxillary} is equivalent to the expression
\begin{align}\label{eq:auxillary3}
\tilde{\calL}^{(adv)}_{\lambda}(f)=\frac{1}{n}\sum_{i=1}^n\Big(|y_i-\langle f,K_{\bx_i}\rangle|+\epsilon\|f^T\bT_{\bx_i}\|\Big)^2+\lambda\|f\|_{\calH}^2,
\end{align}
and we will evaluate the directional derivative at direction $h$ for both $h\in \Span(\bK_{\bX})^\perp$ and $h\in \Span(\bK_{\bX})$. 

For $h\in \Span(\bK_{\bX})^\perp$, we have $|y_i-\langle f,K_{\bx_i}+\bT_{\bx_i}\bdelta\rangle|=0$ and $\tilde{\calL}^{(adv)}_{\lambda}=\frac{\epsilon^2}{n}\sum_{i=1}^n\|f^T\bT_{\bx_i}\|^2+\lambda\|f\|_{\calH}^2(f)$ for $f\in\Span(g_{\lambda}',h)$. Then by the definition of \eqref{eq:auxillary3_lemma}, we have that the directional derivative at $g_{\lambda}'$ and direction $h$ is zero.


For $h\in \Span(\bK_{\bX})$,  the directional derivative of $\tilde{\calL}^{(adv)}$ at $g_{\lambda}$ and direction $h$ is given by
\[
\frac{2}{n}\sum_{i=1}^n\Big(\epsilon|g^T\bT_{\bx_i}|\Big)\Big(|h^TK_{\bx_i}|+\epsilon\frac{h^T\bT_{\bx_i}\bT_{\bx_i}^T g}{|g^T\bT_{\bx_i}|}\Big) + 2\lambda g^Th,
\] 
which is positive if
\[
\frac{2}{n}\sum_{i=1}^n|g^T\bT_{\bx_i}| |h^TK_{\bx_i}|\geq \frac{2}{n}\sum_{i=1}^n\epsilon h^T\bT_{\bx_i}\bT_{\bx_i}^T g+ 2\frac{\lambda}{\epsilon} g^Th. 
\]
Considering that its LHS is in the order of $O(1)$, this is true if $g=g_{\lambda}'$ and $\lambda=o(\epsilon)$. As the derivative of $\tilde{\calL}^{(adv)}_{\lambda}$ at $g_{\lambda}'$ is zero, and $\tilde{\calL}^{(adv)}_{\lambda}$ is convex, \eqref{eq:adv_equivalency} is proved. 
%

Second, we will prove that $\tilde{f}^{(adv)}_{\lambda}$ and $\hat{f}^{(adv)}_{\lambda}$ are close. In the following we will simplify the notations and write $g$ instead of $g_{\lambda}'$.
The proof will be based on the expressions 
\begin{align}\label{eq:adv_equivalency2}
\text{
$\grad \hat{\calL}^{(adv)}_{\lambda}(g)=-\frac{2}{n}\sum_{i=1}^n \Big(y_i-g\big(\bx_i+\bdelta_i(g)\big)\Big)K_{\bx_i+\bdelta_i(g)}+2\lambda g$},\\
\grad \tilde{\calL}^{(adv)}_{\lambda}(g)=-\frac{2}{n}\sum_{i=1}^n(y_i-\langle g,K_{\bx_i}+\bT_{\bx_i}\bDelta_i'(g)\rangle)K_{\bx_i}\bDelta_i'(g)+2\lambda g\label{eq:adv_equivalency3}
\end{align}
where $\bdelta_i(f)=\arg\max_{\bdelta: \|\bdelta\|\leq\epsilon}|y_i-f(\bx_i+\bdelta)|$, and $\bDelta_i'(f)=\arg\max_{\bdelta: \|\bdelta\|\leq\epsilon}|y_i-\langle f, K_{\bx_i}+\bT_{\bx_i}\bdelta\rangle)|$.

The proofs of \eqref{eq:adv_equivalency2} and \eqref{eq:adv_equivalency3} are similar. Note that for
\[
\mathcal{L}(f,\{\bDelta_i\}_{i=1}^n)=\frac{1}{n}\sum_{i=1}^n(y_i-f(\bx+\bdelta_i))^2+\lambda\|f\|_{\calH}^2,
\]
we have $\grad_{\bDelta}\mathcal{L}(g,\{\bDelta_i(g)\}_{i=1}^n)=0$; and by continuity we have $\grad_{\bDelta}\mathcal{L}(g+h,\{\bDelta_i(g)\}_{i=1}^n)=O(\|h\|)$. Since $\|\bDelta_i(g+h)-\bDelta_i(g)\|=O(h)$, we have
\[
\mathcal{L}(g+h,\{\bDelta_i(g+h)\}_{i=1}^n)-\mathcal{L}(g+h,\{\bDelta_i(g)\}_{i=1}^n)=O(\|h\|^2).
\]
So the directional derivative of $\hat{\calL}_{\lambda}^{(adv)}$ at $g$ with direction $h$ is given by
\begin{align*}
&\lim_{t\rightarrow 0} \frac{ \tilde{\calL}^{(adv)}_{\lambda}(g+th)- \tilde{\calL}^{(adv)}_{\lambda}(g)}{t}\\=&
\lim_{t\rightarrow 0}\frac{\mathcal{L}(g+th,\{\bDelta_i(g+th)\}_{i=1}^n)-\mathcal{L}(g+th,\{\bDelta_i(g+th)\}_{i=1}^n)}{t}\\=&\lim_{t\rightarrow 0}\frac{\mathcal{L}(g+th,\{\bDelta_i(g+th)\}_{i=1}^n)-\mathcal{L}(g+th,\{\bDelta_i(g)\}_{i=1}^n)}{t},
\end{align*}
which then verifies \eqref{eq:adv_equivalency2}. The proof of \eqref{eq:adv_equivalency3} is similar.
%
%
%

By definition, we have $\|g\|\leq O(1)$ and
\begin{equation}
\max_{\bdelta: \|\bdelta\|\leq\epsilon}(y_i-\langle g,K_{\bx_i}+\bT_{\bx_i}\bdelta\rangle)^2=\langle g,\bT_{\bx_i}\bdelta_i'(g)\rangle^2 \leq \|g\|^2\|\bT_{\bx_i}\|^2\|\bdelta_i'(g)\|^2= O(\epsilon^2).\label{eq:temp3}
\end{equation}

In addition, we have
\begin{align}\label{eq:temp2}
\bdelta_i(g)-\bdelta_i'(g)=O(\epsilon^{1.5}),
\end{align}
which can be proved as follows. WLOG assuming that $\langle g, \bT_{\bx_i}\bdelta_i'(g)\rangle >0 $ and when $\langle g, \bT_{\bx_i} \rangle = O(1)$, then for all $\bdelta \in B(0,\epsilon)\setminus B(\bdelta_i'(g),\epsilon^{1.5})$, we have
\[
\langle g, \bT_{\bx_i}(\bdelta_i'(g)-\bdelta) \rangle > C \epsilon^2.
\]
So $\bdelta_i(g) \neq \bdelta$ as $\langle g, K_{\bx+\bdelta}\rangle\leq \langle g, K_{\bx_i+\bdelta_i'(g)}\rangle$.

We also have
\begin{equation}
\label{eq:temp4}
(y_i-\langle g,K_{\bx_i+\bdelta_i(g)}\rangle)-(y_i-\langle g,K_{\bx_i} +\bT_{\bx_i}\bdelta_i'(g)\rangle)=O(\epsilon^2).
\end{equation}
and the proof follows from the definition: if $(y_i-\langle g,K_{\bx_i+\bdelta_i(g)}\rangle)-(y_i-\langle g,K_{\bx_i} +\bT_{\bx_i}\bdelta_i'(g)\rangle)>C\epsilon^2$, then $(y_i-\langle g,\bT_{\bx_i}\bdelta_i(g)\rangle)>(y_i-\langle g,K_{\bx_i} +\bT_{\bx_i}\bdelta_i'(g)\rangle)$, which is a contradiction to the definition of $\bdelta_i'(g)$. Similarly we do not have $(y_i-\langle g,K_{\bx_i+\bdelta_i(g)}\rangle)-(y_i-\langle g,K_{\bx_i} +\bT_{\bx_i}\bdelta_i'(g)\rangle)<-C\epsilon^2$, and \eqref{eq:temp4} is proved.

Combining \eqref{eq:adv_equivalency2}-\eqref{eq:adv_equivalency3} with $\grad \tilde{\calL}_{\lambda}^{(adv)}(g)=0$, we have
\begin{align}\nonumber
-\grad \hat{\calL}_{\lambda}^{(adv)}(g)=&\frac{2}{n}\sum_{i=1}^n\Big[(y_i-\langle g,K_{\bx_i+\bdelta_i(g)}\rangle)-(y_i-\langle g,K_{\bx_i} +\bT_{\bx_i}\bdelta_i'(g)\rangle)\Big]K_{\bx_i+\bdelta_i(g)}\\&+(y_i-\langle g,K_{\bx_i} +\bT_{\bx_i}\bdelta_i'(g)\rangle)\Big[K_{\bx_i+\bdelta_i(g)} -(K_{\bx_i}+\bT_{\bx_i}\bdelta_i'(g))\Big]\label{eq:adv_deri_diff}
\end{align}
Applying \eqref{eq:temp2}-\eqref{eq:temp4} to the RHS of \eqref{eq:adv_deri_diff}, then we have the estimation
\[
P_{\bK_{\bX}}\grad \hat{\calL}_{\lambda}^{(adv)}(g)=O(\epsilon^2), P_{\bK_{\bX}^\perp}\grad \hat{\calL}_{\lambda}^{(adv)}(g)=O(\epsilon^{2.5}).
\]

On the other hand, we have that $\grad_f \hat{\calL}_{\lambda}^{(adv)}(f)$ have the following property: $\bH_f \hat{\calL}_{\lambda}^{(adv)}(f)\geq \lambda\bI$. As a result, we can show that 
\[
\|P_{\bK_{\bX}^\perp}(\tilde{f}^{(adv)}_{\lambda}-\hat{f}^{(adv)}_{\lambda})\|=O(\epsilon^{2.5}/\lambda):
\]
for any $f$ such that $\|f-g\|\geq O(\epsilon^{2.5}/\lambda)$, the derivative of $\hat{\calL}_{\lambda}^{(adv)}(g)$ along the direction of $g-f$ is nonzero.

Combining it with the fact that $P_{\bK_{\bX}}\bH_f \hat{\calL}_{\lambda}^{(adv)}(f)\geq P_{\bK_{\bX}}\bH_f \hat{\calL}_{\lambda}(f) \geq c\bI$, we have
\[
\|P_{\bK_{\bX}}(\hat{f}^{(adv)}_{\lambda}-g)\|=O(\epsilon^2),
\]
and
\[
\|\hat{f}^{(adv)}_{\lambda}-g\|=O(\epsilon^2+\epsilon^{2.5}/\lambda)=O(\epsilon^{2.5}/\lambda).
\]
\end{proof}

\begin{proof}[Proof of Theorem~\ref{thm:adv}(b)]
It follows from the proof of Theorem~\ref{thm:aug}(b).\end{proof}

\begin{proof}[Proof of Theorem~\ref{thm:simplecase}]
Let $\bx_0=(\bx_1+\bx_2)/2$ be the middle point of $\bx_1$ and $\bx_2$, $r=\|\bx_2-\bx_1\|$, and let the Taylor expansion of $K_{\bx}$ at $\bx=\bx_0$ be $K_{\bx+t\bu}=\bu_0+t\bu_1+t^2\bu_2+t^3\bu_4+O(r^4)$, with $\|\bu_0\|=O(1), \|\bu_1\|=O(r), \|\bu_2\|=O(r^2), \|\bu_3\|=O(r^3)$ by definition. WLOG assume that $y_1=-1$ and $y_2=1$.

The estimator $\hat{f}_0$ is the least-squares solution that satisfies $\langle \hat{f}_0, \bu_0\rangle=0$ $\langle \hat{f}_0, \bu_1+O(r^2)\rangle=1$, and as a result, $\hat{f}_0\approx \frac{P_{\bu_0^\perp}\bu_1}{\|P_{\bu_0^\perp}\bu_1\|^2}$ in the sense that $\|\hat{f}_0- \frac{P_{\bu_0^\perp}\bu_1}{\|P_{\bu_0^\perp}\bu_1\|^2}\|=O(r\|\hat{f}_0\|)$. By definition, we have $\Lip(\hat{f}_0) =\langle\hat{f}_0,\bu_1\rangle + 2|\langle\hat{f}_0,\bu_2\rangle|+O(r^2)=1+O(r)$ and $\MSE(\hat{f}_0)=\Expect_{t\in \mathrm{Uniform[-1,1]}} |\langle\hat{f}_0,t^2\bu_2\rangle|^2=\frac{1}{3}|\langle\hat{f}_0,\bu_2\rangle|^2=O(r^2)$.


  
By the discussion after \eqref{eq:g0}, the estimator $g_0$ is an estimator such that , $\langle g_0,\bu_0+r^2\bu_2\rangle=0$ and $\langle g_0,r\bu_1+r^3\bu_3\rangle=1$; 
 and $g_0$ is orthogonal to the tangent space at $\bT_{\bX}$, which gives $ \langle g_0,\bu_1+2r\bu_2+3r^2\bu_3\rangle=
\langle g_0,\bu_1-2r\bu_2+3r^2\bu_3\rangle=0$. Combing these four constraints, we have and as a result, $\langle g_0,\bu_0\rangle=\langle g_0,\bu_2\rangle=0$, $\langle g_0,r\bu_1\rangle=1.5$, $\langle g_0,r^3\bu_3\rangle=-0.5$. Then $\|g_0\|=O(r^{-3})$ and
\[
\Lip(g_0)=\max_{-r\leq t\leq r}\langle\bu_1+2t\bu_2+3t^2\bu_2,g_0\rangle\geq \langle\bu_1,g_0\rangle=\frac{3}{2r}.
\]
and
\[
\MSE(g_0)=\Expect_{t\in \mathrm{Uniform[-1,1]}}(1.5t-t^3/2-t)^2=8/105.
\]


On the other hand, for  $\lambda=O(r^2\epsilon^2)$, $g_{\lambda}=\hat{f}_0+\Big(\bP_{\bK_{\bX}^\perp}^T\bSigma\bP_{\bK_{\bX}^\perp}+\frac{\lambda}{\epsilon^2}\bI\Big)^{-1}\bP_{\bK_{\bX}^\perp}^T\bSigma\bP_{\bK_{\bX}}\hat{f}_0$$=\hat{f}_0+\bP_{\Span(\bu_0,
\bu_1)^\perp}\bu_2\bu_2^T\bP_{\Span(\bu_0,
\bu_1)}\hat{f}_0+O(r)$. As a result, when $\lambda=O(r^2\epsilon^2)$, we have $\langle g_{\lambda},\bu_i\rangle=\langle \hat{f}_0,\bu_i\rangle$ for $i=0,1$ and $|\langle g_{\lambda},\bu_2\rangle|<|\langle \hat{f}_0,\bu_2\rangle|$. 
It follows from the previous calculations on $\hat{f}_0$ that 
$\Lip(g_\lambda)/\Lip(\hat{f}_0)<1$ and $\MSE(g_\lambda)<\MSE(\hat{f}_0)<1$.

Applying Theorems~\ref{thm:aug} and~\ref{thm:adv} and the discussions above, Theorem~\ref{thm:simplecase} is proved.
\end{proof}

 \begin{proof}[Proof of Theorem~\ref{thm:quadratic}]
 To prove Theorem~\ref{thm:quadratic}, we use the associated kernel mapping of 
\[
K_{\bx}=[a_1\bx,a_2\bx\bx^T]\in\calH=\reals^{p^2+p}.
\] 
and associated tangent space 
\[
\bT_{\bx}=[a_1\mathbf{e}_i,a_2\bx\odot\mathbf{e}_i]\in\reals^{(p^2+p)\times p},
\]
where $\bx\odot\by=\bx\by^T+\by\bx^T$, and the explicit expression of $g_{\lambda}$ as follows:
\begin{lemma}\label{lemma:quadratic2}
(a) Assuming that $n\leq p, \rank(\bX)=n$, and the SVD of $\bX$ is $\bX=\bU_{\bX}\Sigma_{\bX}\bV_{\bX}^T$.
Then $g_0=\hat{f}_0$ and $\hat{f}_0=\hat{f}_0-(\hat{f}_0^T\bt_1)\bt_2\|\bt_2\|/(\bt_1^T\bt_2)^2$, where
\[
\bt_{1}=\frac{\sum_{i}w_iK_{\bx_i}}{\|\sum_{i}w_iK_{\bx_i}\|},\, 
\text{for}\,\,\bw=(K_{\bX}K_{\bX}^T)^{-1}\mathbf{1}\in\reals^n,
\]
and
\[
\bt_2=[\frac{1}{a_1}\bV_{\bX}\bU_{\bX}^T\mathbf{1},-\frac{1}{a_2}\bV_{\bX}\bF\bV_{\bX}^T], \,\,\text{for $\bF\in\reals^{n\times n}$ defined by $\bF_{ij}=\frac{(\bU_{\bX}^T\mathbf{1})_i(\bU_{\bX}^T\mathbf{1})_j}{\sigma_i+\sigma_j}.$}\,\,
\]
In addition,
\[
g_{\lambda}= (1-a) \hat{f}_0+ a g_0,
\]
where $a=\frac{t\tan^2\theta }{(t\tan^2\theta  +\lambda/\epsilon^2)}$,
where $\theta$ is the angle between $\bt_1$ and $\bT_{\bX}$ and $t=\frac{1}{n}\bt_1^T\sum_{i=1}^n\bT_{\bx_i}\bT_{\bx_i}^T\bt_1^T$.

(b) For this setting of $\bx_i=\be_i$, the formula of $g_{\lambda}$ in Theorem~\ref{thm:quadratic} can be simplified to:

$g_{\lambda}=(1-\theta)g_0+\theta \hat{f}_0$, where $\theta=\theta(\lambda)$ is an increasing function with $\theta(0)=0$ and $\theta(\infty)=1$. In addition, $g_0$ and $\hat{f}_0$ are defined as follows: $\hat{f}_0=\frac{1}{a_1^2+a_2^2}[a_1{\by},a_2\diag({\by})]$, and \[
g_0-\hat{f}_0=\frac{\sum_{i=1}^ny_i}{n(a_1^2+a_2^2)}\Big[-\Big(a_1+\frac{2a_2^2}{a_1}\Big)\mathbf{1}, a_2\bI+\Big(a_2+\frac{a_1^2}{a_2}\Big)\mathbf{1}\mathbf{1}^T\Big].
\]
 \end{lemma}

\textbf{Lipschitz constants} Let first us investigate the Lipschitz constant of $\hat{f}_0=g_{0}$. Applying Lemma~\ref{lemma:quadratic2}(b), the partial derivative of $\hat{f} $is given by
\[
\grad_{\bd}\hat{f}(\bx)=\frac{1}{a_1^2+a_2^2}\sum_{i=1}^n{y}_id_i(a_1+2a_2x_i) 
\]
and the derivative of $g_0-\hat{f}_0$ is given by
\begin{align}\nonumber
&\grad_{\bd}(g_0-\hat{f}_0){f}(\bx)=\frac{\sum_{i=1}^n y_i}{n}\Big\langle\Big[\bd,\bd\bx^T+\bx\bd^T\Big],\Big[\frac{a_1}{a_1^2+a_2^2}\mathbf{1},\frac{a_2}{a_1^2+a_2^2}\bI\Big]-\Big[0,\frac{2}{a_1}\mathbf{1},-\frac{1}{a_2}\mathbf{1} \mathbf{1}^T\Big]\Big\rangle\\\label{eq:directionalderivative}
=&\frac{1}{a_1^2+a_2^2} \frac{\sum_{i=1}^n y_i}{{n}}\Big(-(a_1+2\frac{a_2^2}{a_1})\bd^T\mathbf{1}+2(a_2+\frac{a_1^2}{a_2})(\bd^T\mathbf{1})(\bx^T\mathbf{1})+2a_2\bd^T\bx\Big),
\end{align}

Now let us investigate the Lipschitz constant of $(1-\theta)g_0+\theta \hat{f}_0$, which is the maximal values of $\grad_{\bd}((1-\theta)g_0+\theta \hat{f}_0)(\bx)$ for all $\|\bd\|=1$ and $\|\bx\|\leq 1$.

\textbf{Case 1} When $\theta=0$, by differentiating the directional derivatives with respect to $\bx$ and $\bd$ over the constraint $\|\bd\|=1$ and $\|\bx\|\leq 1$, we know that the largest directional derivative is obtained when $\|\bx\|=1$, $x_i/y_id_i$ is the same for $1\leq i\leq n$, $\|\bd\|=1$ and $d_i/y_i(a_1+2a_2x_i)$ is the same for $1\leq i\leq n$. That is, 
\[
x_i=\frac{y_i^2(a_1+2a_2x_i)}{\sqrt{\sum_{i=1}^n y_i^4(a_1+2a_2x_i)^2}}.
\]
and $\bx$ is in the direction of a linear combination of two vectors, whose $i$-th component is $y_i^2$ and $y_i^2x_i$. When $\max_{i=1}^n y_i^4/\sum_{i=1}^n y_i^4=o(1)$  as $n\rightarrow\infty$, the vector $y_i^2$ is the dominant vector and we have $\|\bx-\tilde{\bx}\|=o(1)$ for $\tilde{\bx}_i=y_i^2/\sqrt{\sum_{i=1}^n y_i^4}$. Similarly, when When $\max_{i=1}^n y_i^2/\sum_{i=1}^n y_i^2=o(1)$, we have $\|\bd-\tilde{\bd}\|=o(1)$ for $\tilde{\bd}_i=y_i/\sqrt{\sum_{i=1}^n y_i^2}$. Then, the  Lipschitz constant  is approximation by
\[
\frac{1}{a_1^2+a_2^2}\Big(\sum_{i=1}^n\frac{y_i^2(a_1+2a_2\frac{y_i^2}{\sqrt{\sum_{i=1}^ny_i^4}})}{\sqrt{\sum_{i=1}^n y_i^2}}\Big)
\]
Ignoring the term $\frac{y_i^2}{\sqrt{\sum_{i=1}^ny_i^4}}$ as it is $o(1)$ as $n\rightarrow\infty$, we have the Lipschitz constant
\[
\Lip(0)=(1+o(1))\frac{a_1}{a_1^2+a_2^2} \Big(\frac{\sum_{i=1}^ny_i^2}{\sqrt{\sum_{i=1}^n y_i^2}}\Big).
\]

\textbf{Case 2} When $\theta=c$ for some small constant $c$, we have that the optimal values are still approximately given by $\bx\sim y_i^2$ and $\bd\sim y_i.$ Then we have $\bx^T\mathbf{1}=O(n)$ and $\bd^T\mathbf{1}=O(\sqrt{n})$, so \eqref{eq:directionalderivative} implies that
\[
\grad_{\bd}(g_0-\hat{f}_0){f}(\bx)\geq 0
\]
and $Lip'(0)<0$.

\textbf{Case 3}  If $\theta=1$, then 
\[
\grad_{\bd} \hat{f}_0(\bx) = \sum_{i=1}^ny_id_i(1+2x_i) + \frac{\sum_{i=1}^n y_i}{n}\Big(4\bd^T\mathbf{1}-2(\bd^T\mathbf{1})(\bx^T\mathbf{1})\Big),
\]
the directional derivative at $\bx= \bd=\mathbf{1}/\sqrt{n}$ is
\[
\frac{\sum_{i=1}^ny_i}{n(a_1^2+a_2^2)}\Big({\sqrt{n}a_1+2a_2}+2(a_1+\frac{2a_2^2}{a_1})\sqrt{n}-2(a_2+\frac{a_1^2}{a_2})n+2a_2\Big).
\]and as $n\rightarrow\infty$, the dominant term is $-\frac{2\sum_{i=1}^ny_i}{a_2}$. Similarly, when $\bx= -\bd=\mathbf{1}/\sqrt{n}$, the dominant term is $\frac{2\sum_{i=1}^ny_i}{a_2}$. As a result, 
\[
\Lip(1)=(1+o(1))\frac{2|\sum_{i=1}^ny_i|}{a_2}.
\]

When $y_i$ are sampled from a distribution with mean $0$ and variance $\sigma^2$, By law of large numbers, $\frac{\sum_{i=1}^ny_i}{\sqrt{n}}$ converges to $N(0,\sigma^2)$ and $\sum_{i=1}^n y_i^2/n$ converges to $\sigma^2$. Since $\frac{a_1}{a_1^2+a_2^2}\leq \frac{1}{4a_1}$,  we have $
\Pr(0.5\Lip(1)>\Lip(0))\rightarrow\Psi(4)-\Psi(-4)\geq 1-6.33 \times 10^{-5}$ as $n\rightarrow\infty$ (here $\Psi$ represents the standard normal CDF).

The rest of the proof of \eqref{eq:lip_quadratic} follows from Theorems~\ref{thm:aug} and~\ref{thm:adv} and the fact $\Span(\bK_{\bX},\bT_{\bX})$ is the full space.

\textbf{Mean Squared Error}
We first present the explicit formula of $\Expect_{\bx}(f(\bx)-f^*(\bx))^2$ when $f-f^*\in\calH$ and $\bx$ is from a spherically symmetric distribution:
\begin{lemma}[Generalization error]
Assuming that $\bx$ is a random vector sampled from a spherically symmetric distribution in $\reals^p$, and $g\in\calH$ is represented by $[\bd,\bD]$, then we have
\[
\Expect_{\bx\sim \mu}(g(\bx))^2=C_{\mu,1}\|\bd\|^2+{C_{\mu,2}} \Tr(\bD^2)+C_{\mu,1}^2\Tr^2(\bD).
\]
Here $C_{\mu,1}=\Expect|\bx_1|^2=\Expect \|\bx\|^2/p$ and $C_{\mu,2}=\Expect|\bx_1|^4-(\Expect|\bx_1|^2)^2$. 
\end{lemma}
\begin{proof}
It follows from 
$\Expect_{\bx\sim N(0,c\bI)}(g(\bx))^2=\Expect (\bx^T\bd+\bx^T\bD\bx)^2=\Expect(\bx^T\bd)^2+\Expect(\bx^T\bD\bx)^2=c\|\bd\|^2+{C_{\mu,2}} \Tr(\bD^2)+{C_{\mu,1}^2}\Tr^2(\bD)$.
\end{proof}

In this case, $g_0$ is given by $g_0(\bx)=\frac{1}{a_1^2+a_2^2}\langle
K_{\bx}, [a_1\by, a_2\diag(\by)]\rangle=\frac{\sum_{i=1}^n (a_1^2x_i+a_2^2x_i^2)y_i}{a_1^2+a_2^2}=\frac{\sum_{i=1}^n (a_1^2x_i+a_2^2x_i^2)\beta_i}{a_1^2+a_2^2}$, and
\[
g_0(\bx)-f^*(\bx)=\frac{a_2^2}{a_1^2+a_2^2}{\sum_{i=1}^n (x_i^2-x_i)\beta_i}.
\]
That is, $g_0-f^*$ is represented by $\frac{a_2^2}{a_1^2+a_2^2}[-\beta, \diag(\beta)]$, and as a result,
\begin{equation}\label{eq:bias1}
\mathrm{Bias}(g_0)=\Big(\frac{a_2^2}{a_1^2+a_2^2}\Big)^2\Big((C_{\mu,1}+C_{\mu,2})\sum_{i=1}^n\beta_i^2+C_{\mu,1}^2(\sum_{i=1}^n\beta_i)^2\Big).
\end{equation}

On the other hand, we have 
\[
g_0(\bx)-\hat{f}_0(\bx)=\frac{\beta^T\mathbf{1}}{n(a_1^2+a_2^2)}\Big(-(a_1^2+2{a_2^2})\bx^T\mathbf{1}+(a_1^2+{a_2^2})(\bx^T\mathbf{1})^2+a_2^2\sum_{i=1}^nx_i^2\Big)
\]
and $g_0-\hat{f}_0$ is represented by $\frac{\beta^T\mathbf{1}}{n(a_1^2+a_2^2)}[-(a_1^2+2{a_2^2})\mathbf{1}, (a_1^2+{a_2^2})\mathbf{1}\mathbf{1}^T+a_2^2\bI]$.  Then it can be verified that 
\begin{align}\nonumber
&\frac{\di}{\di \theta}\mathrm{Bias}((1-\theta)g_0+\theta \hat{f}_0)\Big|_{\theta=0} = -2C_{\mu,1} \frac{\beta^T\mathbf{1}}{n(a_1^2+a_2^2)}\Big\langle-(a_1^2+2{a_2^2})\mathbf{1}, - \frac{a_2^2}{a_1^2+a_2^2} \beta\Big\rangle\\&-2C_{\mu,1}^2\Big\langle\frac{a_2^2}{a_1^2+a_2^2}\diag(\beta), \frac{\beta^T\mathbf{1}}{n(a_1^2+a_2^2)}(a_1^2+2{a_2^2})\bI\Big\rangle -2C_{\mu,2} \frac{a_2^2\beta^T\mathbf{1}}{a_1^2+a_2^2}\frac{(a_1^2+2a_2^2)\beta^T\mathbf{1}}{(a_1^2+a_2^2)}\leq 0.\label{eq:bias2}
\end{align}

As for the case $\theta=1$, we have
\begin{align}\nonumber
&\mathrm{Bias}(\hat{f}_0)-\mathrm{Bias}(g_0)= - \frac{\di}{\di \theta}\mathrm{Bias}((1-\theta)g_0+\theta \hat{f}_0)\Big|_{\theta=0}+C_{\mu,1}\|\frac{\beta^T\mathbf{1}}{n(a_1^2+a_2^2)}(a_1^2+2{a_2^2})\mathbf{1}\|^2\\&+C_{\mu,2}\Tr\Big(\frac{\beta^T\mathbf{1}}{n(a_1^2+a_2^2)}(a_1^2+2{a_2^2})\bI\Big)^2+C_{\mu,1}^2\Big(\frac{(a_1^2+2a_2^2)\beta^T\mathbf{1}}{(a_1^2+a_2^2)}\Big)^2,\label{eq:bias3}
\end{align}
which is always nonnegative. Combining \eqref{eq:bias1}, \eqref{eq:bias2}, \eqref{eq:bias3},  Theorems~\ref{thm:aug} and~\ref{thm:adv}, \eqref{eq:mse_quadratic} is proved.
\end{proof}

\begin{proof}[Proof of  Lemma~\ref{lemma:quadratic2}(a)]
We first prove the following:
\textit{
When $\dim(\Span(\{\bx_i\}_{i=1}^n))=n$, then $\dim(T_{\bX}\cap K_{\bX})=n-1$, and it is given by $T_{\bX}\cap K_{\bX}=\{[\sum_{i}c_i\bx_i,\sum_{i}c_i\bx_i\bx_i^T]: \sum_i c_i=0\}$, and $(T_{\bX}\cap K_{\bX})^\perp \cap K_{\bX}=\Span(\bt_{1})$.}

The fact that $\dim(T_{\bX}\cap K_{\bX})\geq n-1$ follows from \[
K_{\bx}-K_{\bx}=[ a_1(\bc_1+\bc_2),a_2(\bx\bc_1^T+\bc_1\bx^T+\by\bc_2^T+\bc_2\by^T)], \bc_1=\bc_2=(\bx-\by)/2,
\]
which implies $T_{\bX}\cap K_{\bX}=\{[\sum_{i}c_i\bx_i,\sum_{i}c_i\bx_i\bx_i^T]: \sum_i c_i=0\}$, $\dim(T_{\bX}\cap K_{\bX})\geq n-1$, and $\bt_1\perp T_{\bX}\cap K_{\bX}$. Note that by definition, $\dim(K_{\bX})=n$, so we have  $(T_{\bX}\cap K_{\bX})^\perp \cap K_{\bX}=\Span(\bt_{1})$.

Second, we claim that
$T_{\bX}+  K_{\bX}$ is a subspace of dimension $n+(n+1)/2$ if $p\geq n$ and can be described by $T_{\bX}+  K_{\bX}=\{[\by,\bY]:  \bY=\bY^T, \Span(\by)\in\Span(\bx), P_{\Span(\bX)^\perp}\bY=0, P_{\Span(\bX)^\perp}\by=0\}$. Its argument is as follows: first of all, note that $K_{\bX}+T_{\bX}$ includes $[\bx_i,\bx_i^T\bx_i^T]$ and $[\bx_i,2\bx_i^T\bx_i^T]$, it contains $[\bx_i,0]$ and $[\bx_i\bx_i^T]$. Then, we have that it contains $[\by,0]$ for any $\by\in\Span(\bX)$, and any $[0,\bx_i\odot\by]$. The claim is then proved.

Third, we will prove that $T_{\bX}^\perp\cap (T_{\bX}+  K_{\bX})=\Span(\bt_2).$ note that
\[
T_{\bX}=\{[a_1\sum_{i}\bx_i, a_2\sum_i \bx_i\odot\bx_i]:\bx_i\in\reals^p\}=
\{[a_1\bY^T\mathbf{1},a_2(\bX^T\bY+\bY^T\bX)]:\bY\in\reals^{n\times p}\},
\]
which is equivalent to
\[
\Big[a_1\bV_{\bX}\tilde{\bY}^T(\bU_{\bX}^T\mathbf{1}),a_2\bV_{\bX}\Big(\Sigma_{\bX}\tilde{\bY}+\tilde{\bY}^T\Sigma_{\bX}\Big)\bV_{\bX}^T\Big], \,\,\text{where}\,\,\,\tilde{\bY}=\bU_{\bX}^T\bY\bV_{\bX}
\]
Representing the second and the third components as $\bd$ and $\bD$, then we have
\begin{align*}
&\frac{1}{a_1}(\bU_{\bX}^T\mathbf{1})^T\bV_{\bX}^T\bd=(\bU_{\bX}^T\mathbf{1})^T\tilde{\bY}(\bU_{\bX}^T\mathbf{1})= \langle\Sigma_{\bX}\bF+\bF^T\Sigma_{\bX}, \tilde{\bY}\rangle=\langle\bF, \Sigma_{\bX}\tilde{\bY}+\tilde{\bY}^T\Sigma_{\bX} \rangle\\=&\frac{1}{a_2}\langle\bV_{\bX}\bF\bV_{\bX}^T, \bD \rangle.
\end{align*}
As a result, $T_{\bX} = [\frac{1}{a_1}\bV_{\bX}\bU_{\bX}^T\mathbf{1},-\frac{1}{a_2}\bV_{\bX}\bF\bV_{\bX}^T]^\perp\cap(T_{\bX}+  K_{\bX})$.

%

The formula of $a$ follows from the observation that 
\[
a=\frac{(P_{\bK_{\bX}^\perp}\frac{1}{n}\sum_{i=1}^n\bT_{\bx_i}\bT_{\bx_i}^TP_{\bK_{\bX}^\perp}+\lambda \bI/\epsilon^2)^{-1}P_{\bK_{\bX}^\perp}\frac{1}{n}\sum_{i=1}^n\bT_{\bx_i}\bT_{\bx_i}^TP_{\bK_{\bX}}}{(P_{\bK_{\bX}^\perp}\frac{1}{n}\sum_{i=1}^n\bT_{\bx_i}\bT_{\bx_i}^TP_{\bK_{\bX}^\perp})^{-1}P_{\bK_{\bX}^\perp}\frac{1}{n}\sum_{i=1}^n\bT_{\bx_i}\bT_{\bx_i}^TP_{\bK_{\bX}}}.
\]
\end{proof}

 \begin{proof}[Proof of Lemma~\ref{lemma:quadratic2}(b)]
The formula for $g_0=\hat{f}_0$ follows directly from its definition.

On the other hand, $g_0=\hat{f}_0$ has the decomposition $\hat{f}^{(1)}+\hat{f}^{(2)}$, where $\hat{f}^{(1)}$ is the least squares solution to the system that $\langle f,  \bK_{\bX}\bc\rangle =\by^T\bc$ for any $\bc^T\mathbf{1}=0$, and $\hat{f}^{(1)}=\frac{1}{a_1^2+a_2^2}[a_1\bar{\by},a_2\diag(\bar{\by})]$, where $\bar{\by}=\by-\mathrm{Ave}(\by)$. And $\hat{f}^{(2)}$ is in the direction of $K_{\bx_1}+\cdots+K_{\bx_n}=[a_1\mathbf{1},a_2\bI]$ and satisfies $\langle \hat{f}^{(2)}, [a_1\mathbf{1},a_2\bI]\rangle =\sum_{i=1}^n y_i$. That is, $\hat{f}^{(2)}=c_0[a_1\mathbf{1},a_2\bI]$ for $c_0=\frac{\sum_{i=1}^ny_i}{n(a_1^2+a_2^2)}$.

In comparison, $\hat{f}_0$ has the decomposition $\hat{f}^{(1)}+\widetilde{f}^{(2)}$, where $\widetilde{f}^{(2)}=c_0\bt_2$ for $c_0\in\reals$ and $\bt_2=[\frac{1}{a_1}\mathbf{1},-\frac{1}{2a_2}\mathbf{1}\mathbf{1}^T]$ and satisfies $\langle \widetilde{f}^{(2)}, \bt_2\rangle =\sum_{i=1}^n y_i=\langle \hat{f}^{(2)}, [a_1\mathbf{1},a_2\bI]\rangle $. That is, $\widetilde{f}^{(2)}
= 2(a_1^2+a_2^2)c_0[\frac{1}{a_1}\mathbf{1},-\frac{1}{2a_2}\mathbf{1}\mathbf{1}^T]$.
 \end{proof}

\begin{proof}[Proof of Theorem~\ref{thm:general}]
We note that $\|\hat{f}_0\|_{\calH}=\|P_{\bK_{\bX}}f^*\|_{\calH}\leq \|f^*\|_{\calH}$ is bounded. Assuming that the Lipschitz constant of $\hat{f}_0$ is the directional derivative at $\bx_0$ with direction $\bd$ ($\|\bd\|=1$). As $n\rightarrow\infty$, there exists some $c>0$ such that $\|\bx_i - \bx_0-c\bd\|\leq C_{\mu,1}^2$  and $\|\bx_j- \bx_0\|\leq C_{\mu,1}^2$. WLOG assume $i=1$ and $j=2$, and $\Lip(\hat{f}_0)=\frac{\hat{f}_0(\bx_2)-\hat{f}_0(\bx_1)}{\|\bx_1-\bx_1\|}+O(c)=\frac{f^*(\bx_2)-f^*(\bx_1)}{\|\bx_1-\bx_1\|}+O(c)$. 

Next, the proof of Theorem~\ref{thm:simplecase} implies that the directional derivative of $g_0$ at $\frac{\bx_1+\bx_2}{2}$ along the direction of $\bx_2-\bx_1$ is given by $\frac{3}{2}\frac{f^*(\bx_2)-f^*(\bx_1)}{\|\bx_1-\bx_1\|}$. Considering that $g_0$ and $\hat{f}_0$ are approximation of $\hat{f}^{(adv)}_{\lambda_1}$ and $\hat{f}^{(adv)}_{\lambda_2}$ respectively, Theorem~\ref{thm:general} is proved.\end{proof}

\begin{proof}[Proof of Remark~\ref{remark:functionalnorm_general}]

In the proof, we order  $\bx_1,\cdots,\bx_n$ such that the sequence is nondecreasing. As there $\bx_1,\cdots,\bx_n$ are i.i.d. sampled from $\mathrm{Uniform}[0,1]$, there exists a $\bx_i$ such that $|\bx_i-\bx_0|=0(\log n/n)$. Assume that $i=1$ (the proof would hold for other values of $i$ as well), then we may approximate $f^*=K_{\bx_0}$ with 
\[
h=\frac{m-1}{m}K_{\bx_1}-\frac{1}{n}\sum_{i=2}^n K_{\bx_2} =\sum_{i=1}^{n-1}\frac{n-i}{n}(K_{\bx_i}-K_{\bx_{i+1}})
\]
and
\[
\tilde{h}=\sum_{i=1}^{n-1}\frac{n-i}{n}\bT_{\bx_{i+1}}(\bx_{i}-\bx_{i+1}).
\]
Clearly we have $\tilde{h}$ is in the span of $\bT_{\bX}$. In addition, 
\[
\|h-\tilde{h}\|\leq \sum_{i=1}^{n-1}\frac{n-i}{n}\Big\|K_{\bx_i}\!-\!K_{\bx_{i+1}}\!-\!\bT_{\bx_{i+1}}(\bx_{i}-\bx_{i+1})\Big\|\!\leq\! C \|K_{\bx_i}-K_{\bx_{i+1}}\|^2 \sum_{i=1}^{n-1}\frac{n-i}{n}
\leq O(\log^2n/n),
\]
where $\bx_{i+1}-\bx_{i}\leq O(\log n/n)$ as $\bx_1,\cdots,\bx_n$ are nondecreasing and the set is sampled from $\mathrm{Uniform}[0,1]$. 

Then 
\[
\|f^*-\tilde{h}\|\leq \|K_{\bx_1}-K_{\bx_0}\|+\|K_{\bx_1}-h\|+\|h-\tilde{h}\|,
\]
and note that $\tilde{h}$ is in the span of $\bT_{\bX}$, we have 
\[
\frac{\|g_0\|}{\|f_0\|}\geq \frac{\|\hat{f}_0-\tilde{h}\|}{\|\hat{f}_0\|}.
\]
The rest follows.
As a result, as long as $\sigma\sqrt{m}$ goes to infinity and $\sigma=o(1)$, then we can choose $k=\sqrt{m}\sigma^{1/4}$ such that $\|f^*-h\|\rightarrow 0$.
\end{proof}

\bibliographystyle{siamplain}
\bibliography{references}

\end{document}

%% file: ex_shared.tex
\usepackage{lipsum}
\usepackage{amsfonts}
\usepackage{graphicx}
\usepackage{epstopdf}
\usepackage{algorithmic}
\ifpdf
  \DeclareGraphicsExtensions{.eps,.pdf,.png,.jpg}
\else
  \DeclareGraphicsExtensions{.eps}
\fi

\usepackage{enumitem}
\setlist[enumerate]{leftmargin=.5in}
\setlist[itemize]{leftmargin=.5in}

\newsiamremark{remark}{Remark}
\newsiamremark{hypothesis}{Hypothesis}
\crefname{hypothesis}{Hypothesis}{Hypotheses}
\newsiamthm{claim}{Claim}

\headers{Understanding Overfitting in Adversarial Training}{T. Zhang and K. Li}

\title{Understanding Overfitting in Adversarial Training via Kernel Regression}

\author{Teng Zhang\thanks{Department of Mathematics, Univerisity of Central Florida, Orlando, FL 
  (\email{teng.zhang@ucf.edu}).}
\and Kang Li\thanks{Department of Mathematics, Univerisity of Central Florida, Orlando, FL 
  (\email{kang.li@ucf.edu})}
}
\usepackage{amsopn}
\DeclareMathOperator{\diag}{diag}


%% file: Defs.tex
\newcommand{\reals}{{\mathbb R}}

\newcommand{\Span}{\mbox{Span}}

\newcommand{\Tr}{\mbox{Tr}}

\newcommand{\Expect}{\mathrm{Expect}}

\newcommand{\argmin}{\mbox{argmin}}

\newcommand{\bSigma}{\mathbf{\Sigma}}

 \newcommand{\Lip}{\mathrm{Lip}}

    \newcommand{\MSE}{{\mathrm{MSE}}}

\newtheorem{thm}{Theorem}

\newcommand{\grad}{\nabla}

\newcommand{\bL}{\mathbf{L}}

\newcommand{\rank}{\mathrm{rank}}

\newcommand{\bDelta}{\mathbf{\Delta}}
\newcommand{\di}{{\,\mathrm{d}}}

\newcommand{\bc}{\mathbf{c}}
\newcommand{\bd}{\mathbf{d}}

\newcommand{\bt}{\mathbf{t}}
\newcommand{\bu}{\mathbf{u}}

\newcommand{\bw}{\mathbf{w}}
\newcommand{\bx}{\mathbf{x}}
\newcommand{\by}{\mathbf{y}}
\newcommand{\bz}{\mathbf{z}}

\newcommand{\bA}{\mathbf{A}}
\newcommand{\bB}{\mathbf{B}}
\newcommand{\bC}{\mathbf{C}}
\newcommand{\bD}{\mathbf{D}}

\newcommand{\bF}{\mathbf{F}}

\newcommand{\bH}{\mathbf{H}}
\newcommand{\bI}{\mathbf{I}}

\newcommand{\bK}{\mathbf{K}}

\newcommand{\bP}{\mathbf{P}}
\newcommand{\bQ}{\mathbf{Q}}

\newcommand{\bT}{\mathbf{T}}
\newcommand{\be}{\mathbf{e}}
\newcommand{\bU}{\mathbf{U}}
\newcommand{\bV}{\mathbf{V}}
\newcommand{\bW}{\mathbf{W}}
\newcommand{\bX}{\mathbf{X}}
\newcommand{\bY}{\mathbf{Y}}
\newcommand{\bZ}{\mathbf{Z}}
\newcommand{\bdelta}{\mathbf{\delta}}
\newcommand{\btheta}{\mathbf{\theta}}

\newcommand{\calH}{{\cal H}}

\newcommand{\calL}{{\mathcal{L}}}

\newcommand{\calS}{{\mathcal{S}}}

\long\def\ignore#1{}

